\documentclass{article}
\usepackage{iclr2027_conference,times}

\usepackage[utf8]{inputenc} 
\usepackage[T1]{fontenc}    
\usepackage{hyperref}       
\usepackage{url}            
\usepackage{booktabs}       
\usepackage{amsfonts}       
\usepackage{nicefrac}       
\usepackage{microtype}      
\usepackage{xcolor}         

\usepackage{titletoc}

\usepackage{amsmath}
\usepackage{amssymb}
\usepackage{mathtools}
\usepackage{amsthm}
\usepackage{bm}
\usepackage{enumitem}

\usepackage{tikz}
\usetikzlibrary{positioning, arrows.meta, fit, backgrounds, calc, shadows, decorations.pathreplacing}

\usepackage{blkarray}
\usepackage{caption}
\usepackage{multirow}

\usepackage{float} 

\theoremstyle{plain}
\newtheorem{theorem}{Theorem}[section]
\newtheorem{proposition}[theorem]{Proposition}
\newtheorem{lemma}[theorem]{Lemma}
\newtheorem{corollary}[theorem]{Corollary}
\theoremstyle{definition}
\newtheorem{definition}[theorem]{Definition}
\newtheorem{assumption}[theorem]{Assumption}
\theoremstyle{remark}
\newtheorem{remark}[theorem]{Remark}

\definecolor{InternalLink}{HTML}{001473}
\definecolor{CitationLink}{HTML}{007A87}
\definecolor{ExternalLink}{HTML}{8A3B72}

\hypersetup{
    colorlinks=true,
    linkcolor=InternalLink,
    citecolor=CitationLink,
    urlcolor=ExternalLink
}

\usepackage[most]{tcolorbox}

\definecolor{acadblue}{RGB}{238,244,255}

\tcbset{
    blue_style/.style={
        enhanced,
        colback=acadblue,
        colframe=black,
        boxrule=0.5pt,
        arc=2pt,
        auto outer arc,
        boxsep=0pt,
        left=10pt,
        right=10pt,
        top=2pt,
        bottom=2pt,
        before skip=10pt,
        after skip=10pt
    }
}

\definecolor{acadred}{RGB}{250, 235, 235}
\tcbset{
    red_style/.style={
        colback=acadred,       
        colframe=black,         
        boxrule=0.5pt,          
        arc=1pt,                
        auto outer arc,
        left=10pt, right=10pt, top=5pt, bottom=5pt, 
        enhanced,
        before skip=10pt,       
        after skip=10pt         
    }
}

\definecolor{deepshap}{HTML}{B5338A}
\definecolor{kernelshap}{HTML}{4961C0}
\definecolor{treehfd}{HTML}{001AFF}
\definecolor{treeshap}{HTML}{FF5500}
\definecolor{ebmcolor}{HTML}{0F0986}
\definecolor{namcolor}{HTML}{A00B15}

\title{Generalized Functional ANOVA: \\ A Complete Theoretical Framework} 

\author{%
\makebox[\textwidth][c]{%
  \fontsize{9pt}{11pt}\selectfont
  \begin{tabular}{
    @{}c
    @{\hspace{0.8em}}c
    @{\hspace{0.8em}}c
    @{\hspace{0.8em}}c@{}
  }
    \textbf{Baptiste Ferrere}\,$^{{1},{2},{\star}}$ &
    \textbf{Nicolas Bousquet}\,$^{{1},{3},{\dagger}}$ &
    \textbf{Fabrice Gamboa}\,$^{{2},{4},{5},{\dagger}}$ &
    \textbf{Jean-Michel Loubes}\,$^{{2},{4},{6},{\dagger}}$
  \end{tabular}%
}
\\[0.9em]
\makebox[\textwidth][c]{%
  \small
  \begin{tabular}{@{}c@{\hspace{2.2em}}c@{\hspace{2.2em}}c@{}}
    $^{1}$\,EDF R\&D &
    $^{2}$\,Université de Toulouse &
    $^{3}$\,Sorbonne Université
    \\[0.25em]
    $^{4}$\,ANITI &
    $^{5}$\,Universidad de Medellín &
    $^{6}$\,INRIA Regalia
  \end{tabular}%
}
\\[0.75em]
\makebox[\textwidth][c]{%
  \footnotesize
  $^{\star}$\,Corresponding author:
  \href{mailto:baptiste.ferrere@edf.fr}
       {\texttt{baptiste.ferrere@edf.fr}}
  \hspace{1.8em}
  $^{\dagger}$\,Equal supervision.
}
}

\iclrfinalcopy

\begin{document}

\maketitle

\begin{abstract}
The functional ANOVA provides a fundamental representation of square-integrable multivariate functions into main effects and higher-order interactions. For independent inputs, the components belong to mutually orthogonal Hilbert subspaces and admit an explicit representation. For dependent inputs, however, the components are only hierarchically orthogonal: although existence and uniqueness results are available, the Hilbert subspaces underlying the generalized decomposition have remained implicit. We resolve this representation problem for continuous inputs supported on a bounded hyperrectangle whose joint density is bounded above and away from zero. We introduce a distribution-adapted family of functions and prove that it forms a Riesz basis of the $L^2$ space, thereby guaranteeing a unique, stable, and unconditionally convergent representation. We then show that, for every coalition of variables, the corresponding block of this basis exactly characterizes the Hilbert subspace containing the functional ANOVA component. Our construction recovers the classical orthogonal decomposition under input independence. As a direct consequence, computing the generalized functional ANOVA reduces to estimating coefficients in an explicit, distribution-adapted basis. Finally, as a \emph{proof of concept}, we introduce an elementary, fast and model-agnostic estimator based on our theoretical results. Experiments on synthetic and real-world datasets illustrate its connections with established tabular explanation methods and show that low order components often capture most of the signal in the model output.
\end{abstract}

\section{Introduction}\label{section:intro}

\paragraph{Notations.}
Let $\mathbb N_+$ denote the set of positive integers and let take $d \in \mathbb N_+$. We denote by $[d]$ the set $\{1,\dots,d\}$. We also denote by $\mathbf X$ (in bold) the real valued random vector $(X_1, \dots, X_d)$ whose distribution is denoted $P_{\mathbf X}$. For any subset $S$ of $[d]$, we define $\mathbf X_S \coloneqq (X_i)_{i \in S}$. Finally, we denote by $L^2(P_{\mathbf X})$ the Hilbert space of $P_{\mathbf X}-$square integrable functions.

The functional ANOVA decomposition provides a fundamental representation of multivariate functions. It expresses any $f \in L^2( P_{\mathbf X} )$ as follows:
\begin{equation}\label{eq:intro-anova}
f(\mathbf X)
=
f_{\emptyset}
+\sum_{i=1}^{d}f_i(X_i)
+\sum_{i<j}f_{i,j}(X_i,X_j)
+\cdots
+f_{[d]}(\mathbf X)
=
\sum_{S\subseteq[d]}f_S(\mathbf X_S).
\end{equation}
Each component $f_S$ depends only on the variables indexed by $S$ and represents the contribution associated with this subset after removing all lower-order effects. When the inputs are mutually independent, these components belong to mutually orthogonal Hilbert subspaces and admit an explicit expression through the Möbius inversion of conditional expectations \citep{Hoeffding1948}. The functional ANOVA therefore provides both a decomposition of individual functions and an orthogonal decomposition of the underlying $L^2(P_{\mathbf X})$ space.

This classical construction fundamentally relies on input independence. When the inputs are dependent, the component subspaces are no longer mutually orthogonal, and the usual conditional-expectation formula does not isolate pure interaction effects. The identification principle now commonly referred to as \emph{hierarchical orthogonality} already appears in the functional ANOVA framework of \citet{stone1994use} and was subsequently made explicit by \citet{huang1998projection}. More precisely, for every $S\subseteq[d]$, the $S$-order component is required to be orthogonal to every function belonging to a space indexed by a strict subset $T\subsetneq S$. Under suitable bounded-density assumptions, these works obtain an essentially unique theoretical ANOVA decomposition. \citet{hooker_2007} subsequently recast this construction for dependent inputs as a constrained functional optimization problem and call it the \emph{generalized} functional ANOVA. Later works investigated its existence and uniqueness under different dependence assumptions \citep{chastaing_generalized_2012,idrissi2025hoeffding}. Nevertheless, the geometry of this generalized decomposition remains implicit: the Hilbert subspaces containing its components have not been explicitly characterized through a basis of $L^2(P_{\mathbf X})$ well suited to the distribution $P_{\mathbf X}$.

This representation problem is directly relevant to statistical learning theory. A full functional ANOVA decomposes an arbitrary multivariate function into components of every order, whereas additive learning imposes a structural restriction by retaining only a small collection of these components. Generalized additive models retain first-order effects \citep{hastie2017generalized}, while GA\textsuperscript{2}Ms additionally include selected pairwise interactions \citep{lou2013accurate}. Related ideas underlie spline-based functional ANOVA \citep{stone1994use}, component-selection methods such as sparse additive models \citep{ravikumar2009sparse,raskutti2012minimax,haris2022generalized}. By exploiting sparsity and low-order interactions, these models provide flexible nonparametric representations while mitigating the curse of dimensionality.

Additive and low-order interaction structures have recently attracted renewed attention in machine learning. \citet{bhattacharya2024deep} established minimax-optimal rates for suitably constructed deep neural networks over nonparametric interaction models of fixed order. At the architectural level, Neural Additive Models \citep{agarwal2021neural}, NODE-GAM \citep{chang2021node}, and Neural Basis Models \citep{radenovic2022neural} combine neural-network estimation with additive functional representations. These developments illustrate the continuing importance of identifying and estimating low-dimensional components of a multivariate function. They also emphasize the need for a precise characterization of the spaces to which these components belong, particularly when the inputs are not independent.

Functional decompositions have also become increasingly important for explainable AI. Inherently interpretable models expose low-order functional components by construction. In parallel, several post-hoc explanation methods are closely connected to the functional ANOVA. Under input independence, ANOVA components can be interpreted as Harsanyi dividends \citep{harsanyi_1963} allowing to compute Shapley attributions \citep{Shapley_1953,lundberg2017unified}. This results from redistributing each interaction component among its participating variables \citep{owen2014sobol,owen2017shapley}. Functional ANOVA has also been used to obtain \emph{pure} interaction effects in additive models and to explain tree ensembles \citep{lengerich2020purifying}. Very recently, \citep{benard2025tree} develops a method from data, to estimate the functional ANOVA of a trained tree ensemble, highlights strong connections with TreeSHAP \citep{lundberg2018consistent} and SHAP-based functional decompositions of tree ensembles \citep{herren2022statistical,bordt2023shapley,muschalik2024shapiq,muschalik2024beyond}. Finally, \citet{ferrere2026exact} show that, for categorical inputs, the generalized functional ANOVA admits an explicit finite-dimensional representation based on inverse likelihood weighting. This result provides a finite-dimensional precursor to the construction studied here. Extending it to continuous inputs is, however, not obvious: the resulting representation becomes infinite-dimensional, and completeness or linear independence alone is insufficient to warrant existence, uniqueness, unconditional convergence, and norm stability. These properties require establishing the substantially stronger Riesz-basis property for the dictionary of functions that gives a representation of the functional
decomposition.

Thus, rather than constituting an isolated explanation method, functional ANOVA provides a common mathematical object connecting inherently additive predictors and post-hoc additive explanations.

More precisely, in this work, we focus on random vectors $\mathbf X \in [-1,1]^d$ whose joint density $p$ is bounded above and away from zero, a standard regularity condition (see for example \citet{stone1977consistent}) and largely used in the analysis of generalized functional ANOVA and nonparametric function estimation \citep{stone1994use,gyorfi2002distribution,chastaing_generalized_2012,benard2025tree}.
We introduce a distribution-adapted family obtained by weighting a Hilbert basis by inverse marginal densities and prove that it forms a Riesz basis of $L^2(P_{\mathbf X})$. We use this basis to characterize explicitly the Hilbert subspace containing every generalized functional ANOVA component and derive its unique and stable expansion. Finally, we introduce a simple and fast model-agnostic explainer as a proof of concept and use it to illustrate the connections between generalized functional ANOVA and established tabular explanation methods. We detail our contribution in the next paragraph.

\paragraph{Contribution.}
\begin{enumerate}[leftmargin=*, itemsep=0.3em]

\item We introduce in Definition~\ref{def:key_def} a novel distribution-adapted family of functions. We prove in Theorem~\ref{theorem:riesz} that this family forms a \emph{Riesz basis} of $L^2(P_{\mathbf X})$, thereby guaranteeing a unique, stable, and unconditionally convergent representation of every function in this Hilbert space.

\item We characterize the Hilbert subspaces underlying the generalized functional ANOVA. Specifically, Theorem~\ref{thm:representation} shows that, for every $S\subseteq[d]$, the component $f_S$ belongs to a fully characterized subspace, denoted by $\mathcal H_S$, and admits a unique expansion in the corresponding block of our Riesz basis. To the best of our knowledge, this is the first explicit characterization of these subspaces for dependent continuous inputs. As a direct consequence, computing the generalized functional ANOVA reduces to a standard statistical learning problem: estimating coefficients in a Riesz basis.

\item We demonstrate in Corollary~\ref{cor:independent} that, under input independence, our construction exactly recovers the classical orthogonal decomposition. In particular, the subspaces $\mathcal H_S$ become mutually orthogonal and the resulting components coincide with the classical functional ANOVA components.

\item As a \emph{proof of concept} for our theoretical results, we propose a simple full data-driven algorithm, without knowing the architecture of the model $f$ or the distribution $P_{\mathbf X}$. Through experiments on synthetic and real-world datasets, we illustrate how the decomposition behaves in practice and highlight its close empirical connections with well-established tabular XAI methods. We further observe that, across the considered datasets and model classes, restricting the decomposition to first- and second-order components is often sufficient to capture most of the variance in the model output.

\end{enumerate}

\section{Preliminaries}\label{section:background}

We will focus on the expansion \eqref{eq:intro-anova}, where each component $f_S( \mathbf X_S)$ depends only on the subset of variables $\mathbf{X}_S$. The constant term $f_\emptyset$ captures the overall mean effect, the functions $f_{i}$ represent the marginal (or main) effects, and higher-order terms $f_S$ with $|S| \geq 2$ encode the pure interaction effects among the variables indexed by $S$. In this paper, we assume that $\mathbf X$ has a joint continuous density $p$, \textit{i.e.} $dP_{\mathbf X}(\mathbf x) / d\mathbf x = p(\mathbf x)$.

\subsection{Generalized Functional ANOVA}

For all $S \subseteq [d]$, let $L^2_S$ be the Hilbert space of square integrable functions of $\mathbf X_S$. One can define the components of \eqref{eq:intro-anova} as the solution of the following optimization problem \citep{hooker_2007}:
\begin{tcolorbox}[blue_style]
\begin{equation}
    \{ f_S \}_{ S \subseteq [d] } \coloneqq \operatornamewithlimits{argmin}_{\{g_S \in L^2_S \}_{S \subseteq [d]}} \int_{\mathbb R^d} \left( \sum_{S \subseteq [d]} g_S(\mathbf{x}_S) - f(\mathbf{x}) \right)^2 p(\mathbf{x}) \, d\mathbf{x},
    \label{eq:fanova_objective}
\end{equation}
subject to the \emph{hierarchical orthogonality constraint}:
\begin{equation}
    \forall T \subsetneq S \subseteq [d], \: \forall \, g_T \in L_T^2 , \:  \int_{\mathbb R^d} f_S(\mathbf{x}_S) \, g_T(\mathbf{x}_T) \, p(\mathbf{x}) \, d\mathbf{x} = 0.
    \label{eq:fanova_orthogonality}
\end{equation}
\end{tcolorbox}
Conceptually, decomposition~\eqref{eq:intro-anova} is sought under the constraint that information is non-redundant: for any subset $T \subseteq [d]$, the contribution of $f_S$ for $T \subsetneq S \subseteq [d]$ must capture only the variance that cannot be explained by any function of a strict subset of its variables $\mathbf{X}_T$. Roughly speaking, each component $f_S$ is \emph{purely an $|S|$-th order effect}, free of any lower-order information.

Under mild non-degeneracy conditions on the joint distribution \citep{chastaing_generalized_2012, idrissi2025hoeffding}, the optimization problem~\eqref{eq:fanova_objective} under the constraint \eqref{eq:fanova_orthogonality} admits a unique solution, and the resulting decomposition generally allocates less variance to higher-order components. Throughout this work, we assume that the support of $\mathbf{X}$ is a hyperrectangle and impose the following regularity condition on the density \citep{stone1977consistent,stone1994use}.

\begin{assumption}\label{assu:densite}
Without loss of generality, we assume that $\mathbf{X}$ is supported on $[-1,1]^d$ and that its density $p$ is bounded above and away from zero: there exist constants $0 < c_1 < c_2$ such that:
\begin{equation}
    \forall\, \mathbf{x} \in [-1,1]^d, \: c_1 \leq p(\mathbf{x}) \leq c_2 < \infty.
\end{equation}
\end{assumption}

\begin{theorem}[\cite{hooker_2007}]
For any function $f$ member of $L^2(P_{\mathbf X})$, there exists a unique collection of functions $\{ f_S \}_{S \subseteq [d]}$ such that $ f(\mathbf{X}) = \sum_{S \subseteq [d]} f_S(\mathbf{X}_S) $ under the hierarchical orthogonality constraint (\ref{eq:fanova_orthogonality}).
\end{theorem}

\subsection{Standard Functional ANOVA}

In this short subsection, we recall the basics results under independence, i.e., when the underlying density $p$ is product. In particular, we have the following theorem due to \citet{Hoeffding1948} and \citet{stone1994use}.
\begin{theorem}[Standard functional ANOVA]
    Under mutual independence, the components $f_S( \mathbf X_S )$ are explicit and are given by the Möbius transform \cite{rota1964foundations} of the conditional expectations: 
\begin{equation}
    \forall S \subseteq [d], \: f_S( \mathbf X_S) = \sum\limits_{T \subseteq S} (-1)^{ \vert S \vert - \vert T \vert } \mathbb{E}\left[ f(\mathbf X) \mid \mathbf X_T \right].
\end{equation}
Moreover, the constraint \eqref{eq:fanova_orthogonality} implies the mutual orthogonality of the decomposition:
\begin{equation}
    \forall S,T \subseteq [d], \: S \neq T \implies \mathbb{E}[f_S(
\mathbf{X}_S) \, f_T(\mathbf{X}_T)] = 0,
\end{equation}
recovering the well-known orthogonal variance decomposition.
\end{theorem}

\paragraph{A strong connection with Shapley values.}
Functional ANOVA is closely connected to Shapley values through Harsanyi dividends \citep{harsanyi_1963}. For a fixed input $\mathbf{x}$, the components $f_S(\mathbf{x}_S)$ are the dividends of the cooperative game with value function $v_{\mathbf{x}}(S) := \sum_{T \subseteq S} f_T(\mathbf{x}_T)$. Distributing each dividend equally among its participating features yields the \emph{fANOVA Shapley values} \citep{owen2014sobol, owen2017shapley}: $\mathrm{shap}^{\mathrm{ANOVA}}_i(\mathbf{x}) \coloneqq \sum_{S \ni i} f_S(\mathbf{x}_S)/|S|$. Under mutual independence of the input features, these attributions coincide with SHAP values for the conditional-expectation game \citep{lundberg2017unified}. In this paper, we further investigate the connection empirically by comparing estimated ANOVA main effects with SHAP-based explanations. The visual comparisons in the main body of the paper and in Appendix \ref{appendix:additional_figures}; and the global feature-importance rank correlations in Appendix~\ref{app:global-importance}, provide complementary qualitative and quantitative evidence of strong overall agreement across the considered datasets and model classes.

\section{Theoretical Characterization}\label{section:main}

In this section, we establish our main theoretical results. We first define a novel family of functions that generalizes the inverse likelihood mechanism of \citet{ferrere2026exact} to continuous inputs. We then prove that this family forms a Riesz basis of $L^2(P_{\mathbf X})$. A Riesz basis (see for example \citet{brezis2011functional}) is the image of an orthonormal basis under a bounded invertible operator; it preserves uniqueness, stability, and unconditional convergence while allowing for non-orthogonality. Finally, we show that this basis is adapted to the generalized functional ANOVA and provides an explicit solution to the optimization problem formulated by \citet{hooker_2007}.

\subsection{The Distribution-Adapted Riesz Basis}

Throughout this section (and the rest of the paper), let first denote by $p_S$ the marginal density of $\mathbf X_S$ for every non-empty set $S \subseteq[d]$. For each $m \in \mathbb{N}_+$, let $\phi_m$ denote the normalized Legendre polynomial of degree $m$, whose definition and properties are recalled in Appendix \ref{appendix:legendre}.

\begin{definition}\label{def:key_def}
First, let $\xi_{\emptyset} \coloneqq 1$. Then, for every set $S \subseteq [d]$, denote $\bm{m}_S := (m_j)_{j \in S} \in \mathbb{N}_+^{\vert S \vert}$. For $\mathbf x \in [-1,1]^d$ we set:
\begin{equation}
    \xi_S^{( \bm{m}_S )}(\mathbf x) := \frac{1}{ \sqrt{ 2^{d - \vert S \vert} } } \cdot \frac{ \prod\limits_{j \in S} \phi_{m_j} ( x_j ) }{p_S(\mathbf x_S)}
\end{equation}
\end{definition}
We denote by $\Xi$ the resulting collection of functions. This construction provides a continuous counterpart to the inverse likelihood mechanism of \citet{ferrere2026exact}: the sign functions used in the categorical setting are replaced by normalized Legendre polynomials, while inverse probability-mass weighting is replaced by inverse marginal-density weighting. In Appendix \ref{appendix:proofs}, we show that there exists an \emph{invertible} and \emph{bounded} operator $\mathcal T$ such that $\Xi$ is the direct image of an orthonormal basis by this operator $\mathcal T$. As a consequence, we have the next representation theorem:
\begin{theorem}[Riesz Basis]\label{theorem:riesz}
The family of functions $\Xi$ introduced in Definition \ref{def:key_def} forms a Riesz basis of $L^2(P_{\mathbf X})$.
\end{theorem}
As a direct consequence, there exist \emph{Riesz constants} $C_{\min},C_{\max}>0$ (which are independent of summation indices $S$ and $\bm m_S$) such that, for every sequence $a\in\ell^2$, the following spectral inequality is satisfied:
\begin{equation}
    C_{\text{min}}\left\| \bm a \right\|_{\ell_2}^2 \leq \left\| \sum\limits_{S \subseteq [d]} \sum_{\bm{m}_S \in \mathbb{N}_+^{|S|}} a_S^{(\bm{m}_S)} \, \xi_S^{(\bm{m}_S)} \right\|_{ L^2( P_{\mathbf X} ) }^2 \leq  C_{\text{max}}\left\| \bm a \right\|_{\ell_2}^2.
\end{equation}
In the very particular case where the basis is an orthonormal (or also said an Hilbert) basis, we have $C_{\text{min}} = C_{\text{max}} = 1$.

\subsection{The Representation Theorem}

Before obtaining the main theorem, we introduce the following Lemma that is the desired consequence of the inverse likelihood mechanism.
\begin{lemma}[Hierarchical Orthogonality]\label{lemma:hierarchical}
    The family of functions $\Xi$ introduced in Definition \ref{def:key_def} satisfies hierarchical orthogonality in $L^2(P_{\mathbf X})$. More formally, we have:
\begin{equation}
    \forall T \subsetneq S \subseteq [d], \forall \bm{m}_S \in \mathbb{N}_+^{\vert S \vert}, \forall \bm{n}_T \in \mathbb{N}_+^{\vert T \vert}, \quad \mathbb{E}\left[ \xi_S^{(\bm{m}_S)}( \mathbf X ) \cdot \xi_T^{(\bm{n}_T)}( \mathbf X ) \right] = 0.
\end{equation}
\end{lemma}

Let Assumption~\ref{assu:densite} hold, thanks to this previous technical result and to the Riesz basis property, we have all the ingredients to state the following theorem, which is the central contribution of the paper.
\begin{tcolorbox}[blue_style]
\begin{theorem}[Main Theorem]\label{thm:representation}
    For any $f \in L^2(P_{\mathbf X})$, there exists a unique sequence of coefficients $\bm{\theta}^{\star} \coloneqq \big(\theta_S^{(\bm{m}_S)}\big)_{S \subseteq [d],\, \bm{m}_S \in \mathbb{N}_+^{|S|}}$ such that the components satisfying~\eqref{eq:fanova_objective} and \eqref{eq:fanova_orthogonality} express as follows:
    \begin{equation}\label{eq:ALD_components}
        f_S = \sum_{\bm{m}_S \in \mathbb{N}_+^{|S|}} \theta_S^{(\bm{m}_S)} \, \xi_S^{(\bm{m}_S)}, \qquad \forall\, S \subseteq [d].
    \end{equation}
\end{theorem}
\end{tcolorbox}
By denoting the sequence of Hilbert spaces $\left( \mathcal{H}_S \right)_{ S \subseteq [d] }$ associated to the ANOVA decomposition of $L^2(P_{\mathbf X})$ (\textit{i.e.} the hierarchical Hilbert spaces that will contain all the functions $f_S$, see \citet{huang1998projection,idrissi2025hoeffding}), this theorem is strictly equivalent to the following characterization:
    \begin{equation}
        \mathcal H_S = \overline{\mathrm{span}}\left\{ \xi_S^{(\bm m_S)} \;\middle|\; \bm m_S \in \mathbb{N}_+^{\lvert S \rvert} \right\}.
    \end{equation}

\subsection{The Independent Setting}

As discussed before, under feature independence, the constraint~\eqref{eq:fanova_orthogonality} strengthens to full mutual orthogonality between components. In what follows, we show that our basis recovers naturally this property.
\begin{corollary}[Orthogonality]\label{cor:independent}
    Let Assumption~\ref{assu:densite} hold. Let also assume that $P_{\mathbf X}$ is a product measure over $[-1,1]^d$, then the elements of $\Xi$ are mutually orthogonal set by set, i.e.
    \begin{equation}
    \forall T \neq S \subseteq [d], \forall \bm{m}_S \in \mathbb{N}_+^{\vert S \vert}, \forall \bm{n}_T \in \mathbb{N}_+^{\vert T \vert}, \quad \mathbb{E}\left[ \xi_S^{(\bm{m}_S)}( \mathbf X ) \cdot \xi_T^{(\bm{n}_T)}( \mathbf X ) \right] = 0.
\end{equation}
\end{corollary}
As a direct consequence of this corollary, the resulting decomposition is mutually orthogonal, which means that:
\begin{equation}
    \forall S,T \subseteq [d], \: S \neq T \implies \mathbb{E}[f_S(
\mathbf{X}_S) \, f_T(\mathbf{X}_T)] = 0,
\end{equation}
so our formulation entirely recovers the \emph{standard} setting. This result is also proved in Appendix \ref{appendix:proofs}. 
\begin{remark}[Uniform Setting]
    In the uniform independent setting, the blocks of $\Xi$ coincide, up to nonzero subset-dependent scaling factors, with the tensorized Legendre basis. Consequently, they recover the same mutually orthogonal decomposition.
\end{remark}

\section{Estimating the Decomposition}\label{sec:estimating}

In this section, we assume that an i.i.d. sample $\{\mathbf{X}^{(i)}\}_{i=1}^{n}$ is at hand together with the evaluations $\{f(\mathbf{X}^{(i)})\}_{i=1}^{n}$. Here, $f$ is treated as a \emph{black-box} model: we can query it pointwise but have no access to its internal architecture (e.g., a trained neural network or tree ensemble). 

\subsection{A Nonparametric Estimation Problem}
Since the basis $\Xi$ is available in closed form, estimating the functional ANOVA decomposition reduces to estimating the coefficients:
\begin{equation}
    \bm{\theta}^{\star} = \left(\left( \theta_S^{(\bm{m}_S)}\right)_{\bm{m}_S \in\mathbb{N}_{+}^{|S|}}\right)_{S \subseteq [d]}.
\end{equation}

A natural estimator is obtained by minimizing the general least-squares risk given by:
\begin{equation}\label{eq:empirical_risk}
    \mathcal{L}_n(\bm \beta) := \frac{1}{n}\sum_{i=1}^{n} \left( f(\mathbf{X}^{(i)}) - \sum_{S \subseteq [d]} \sum_{\bm{m}_S \in \mathbb{N}_+^{|S|}} \beta_S^{(\bm{m}_S)} \cdot \xi_S^{(\bm{m}_S)}(\mathbf{X}^{(i)}) \right)^{2} + \mathrm{pen}( \bm \beta ).
\end{equation}
The Riesz-basis representation turns the computation of the generalized functional ANOVA into the estimation of an infinite-dimensional sequence of coefficients. Abstracting away from the non-orthogonal geometry of the basis, this is a classical problem of nonparametric series estimation: the target function is represented by an infinite coefficient sequence, approximated within a sequence of finite-dimensional spaces, and the corresponding coefficients are estimated by empirical least squares \citep{newey1997convergence,gyorfi2002distribution,chen2007large,tsybakov2009introduction,cohen2013stability}. 

However, developing a complete nonparametric theory for this formulation—including the choice of truncation or regularization, the approximation properties of the resulting sieve spaces, the effect of estimating the distribution-adapted basis, and the derivation of convergence or minimax rates—raises substantial statistical questions that require a separate treatment and lie beyond the scope of this paper. We therefore do not pursue the asymptotic analysis of the estimator here. Instead, in the next subsection, we introduce a deliberately simple finite-dimensional estimator under explicit structural assumptions. This estimator is intended as a \emph{proof of concept} to illustrate the practical and computational implications of our theoretical representation.

\subsection{A Truncated Series Estimator}

To obtain a finite-dimensional problem, we truncate the Riesz expansion jointly with respect to the interaction order and the coordinatewise polynomial degree. This construction follows standard truncation strategies from the polynomial chaos expansion literature \citep{xiu2002wiener,tsybakov2009introduction,blatman2011adaptive}. The restriction on the interaction order is motivated by the \emph{sparsity-of-effects principle} \citep{montgomery2017design}, according to which the response is often primarily driven by main effects and low-order interactions, and by the \emph{reluctance principle} \citep{yu2019reluctant}, which favors lower-order effects whenever they provide comparable explanatory power.

\begin{definition}[Truncation set]
Let \(K\in[d]\) and \(M\in\mathbb N_{+}\). We define the truncation set of maximum interaction order \(K\) and maximum coordinatewise degree \(M\) by
\begin{equation}
    \mathcal I_K^M
    \coloneqq
    \left\{
        (S,\bm m_S)
        \,\middle|\,
        S\subseteq[d],\ 
        |S|\leq K,\ 
        \bm m_S\in\{1,\ldots,M\}^{|S|}
    \right\},
\end{equation}
where the empty set is associated with the unique empty multi-index.
\end{definition}

\begin{remark}[Dimension of the truncated space]
The truncation set contains
\begin{equation}
    \mathcal N(d,K,M)
    \coloneqq
    \sum_{k=0}^{K}\binom{d}{k}M^k
    \leq \underbrace{M^K}_{ \text{exponential complexity} } * 
    \underbrace{\left(\frac{\mathrm e \ d}{K}\right)^K}_{ \text{Sauer-Shelah complexity} }
\end{equation}
elements. Consequently, its dimension is polynomial in \(d\) and \(M\) when \(K\) is fixed, but grows combinatorially with the maximum interaction order.
\end{remark}

This truncation is consistent with the low-order structures underlying GAMs, GA\textsuperscript{2}Ms, and fixed-order interaction models \citep{nori2019interpretml,agarwal2021neural,bhattacharya2024deep,benard2025tree}. It does not, however, enforce sparsity among components of the same order, since every subset \(S\) satisfying \(|S|\leq K\) is retained. Such effect selection could be incorporated through an additional penalty on the coefficient blocks. Here, we deliberately retain the simplest construction and empirically examine whether main and pair effects capture most of the model output.

The basis functions depend on the marginal densities of the input distribution, which are unknown in practice. We estimate these densities using truncated orthogonal-series estimators constructed from normalized Legendre polynomials and denote by \(\widehat{\xi}_S^{(\bm m_S)}\) the resulting plug-in basis functions. Restricting the expansion to \(\mathcal I_K^M\) then yields the finite-dimensional estimator
\begin{equation}\label{eq:estimator_a}
    \widehat{\bm\theta}_n^{K,M}
    \in
    \operatorname*{arg\,min}_{
        \bm\beta\in\mathbb R^{\mathcal N(d,K,M)}
    }
    \frac{1}{n}
    \sum_{i=1}^{n}
    \left(
        f(\mathbf X^{(i)})
        -
        \sum_{(S,\bm m_S)\in\mathcal I_K^M}
        \beta_S^{(\bm m_S)}
        \widehat{\xi}_S^{(\bm m_S)}(\mathbf X^{(i)})
    \right)^2.
\end{equation}
When the empirical design matrix has full column rank, the minimizer is unique. Otherwise, we define \(\widehat{\bm\theta}_n^{K,M}\) as the minimum-Euclidean-norm least-squares solution. In Appendix~\ref{appendix:estimator}, Figure~\ref{fig:global-scheme} summarizes the overall procedure and we also provide the complete estimation pipeline. 

Quantifying how the truncation error and the plug-in estimation of the input densities jointly affect the estimation of the generalized functional ANOVA belongs to the broader statistical analysis identified above is not pursued in this paper.

\section{Experiments}\label{section:exp}

Consistent with the theoretical scope of this paper, the experiments are intended as a \emph{proof of concept} rather than as a comprehensive statistical evaluation of the proposed estimator. They investigate two main questions. (i) Reconstruction fidelity, \textit{i.e.} how good the truncated and estimated decomposition reconstruct the model ; (ii) Agreement with additive explainers, \textit{i.e.} how these components align with standard additive explanation methods.
We first consider an analytically tractable example for which the generalized functional ANOVA components are available in closed form. We display on Figure \ref{fig:AC} theoretical univariate functions and visually compare them with the classic explanation methods. We then conduct experiments on six real-world and classical datasets given in the Table \ref{table:datasets}.

  \begin{table}[ht]
    \centering
    \setlength{\tabcolsep}{5pt}
    \small
    \begin{tabular}{llrr}
      \toprule
      \textbf{ID} & \textbf{Name} & $n$ & $d$ \\
      \midrule
      \textbf{BS} & \emph{Bike Sharing}~\citep{fanaee2014event}              & 17\,379 & 8  \\
      \textbf{CH} & \emph{California Housing}~\citep{pace1997sparse}         & 20\,640 & 8  \\
      \textbf{CI} & \emph{Census Income}~\citep{kohavi1996scaling}     & 48\,842 & 14 \\
      \textbf{EG} & \emph{Electrical Grid}~\citep{arzamasov2018towards} & 10\,000 & 12 \\
      \textbf{PI} & \emph{Pima Indian Diabetes}~\citep{smith1988using}      & 768     & 8  \\
      \textbf{SC} & \emph{Superconductivity}~\citep{hamidieh2018data}        & 21\,263 & 81 \\
      \bottomrule
    \end{tabular}
    \captionof{table}{Datasets characteristics.}
    \label{table:datasets}
  \end{table}

For both analytical and real-world dataset settings, we train two black-box predictors---a neural network and an XGBoost model---using fixed model specifications detailed in Appendix~\ref{appendix:experiments}. We apply our decomposition to both predictors and evaluate the reconstruction fidelity of our low order truncated ANOVA. We additionally visually compare the resulting components with several established explanation methods, using each method with the model class for which it is designed. For neural networks, we consider KernelSHAP \citep{lundberg2017unified}, a widely used model-agnostic attribution method, and DeepSHAP \citep{shrikumar2017learning,lundberg2017unified}, its model-specific counterpart. For tree ensembles, we consider TreeSHAP \citep{lundberg2018consistent}, the standard SHAP implementation for tree-based models, and TreeHFD \citep{benard2025tree}, which directly estimates a functional ANOVA decomposition of trained tree ensembles. Finally, we train Neural Additive Models (NAMs, \citet{agarwal2021neural}) and Explainable Boosting Machines (EBMs, \citet{nori2019interpretml}), apply our decomposition to the resulting predictors. We also evaluate the reconstruction fidelity and visually compare the recovered main effects with the components natively produced by these architectures.

\begin{figure}[t]
  \centering
  \includegraphics[width=\textwidth]{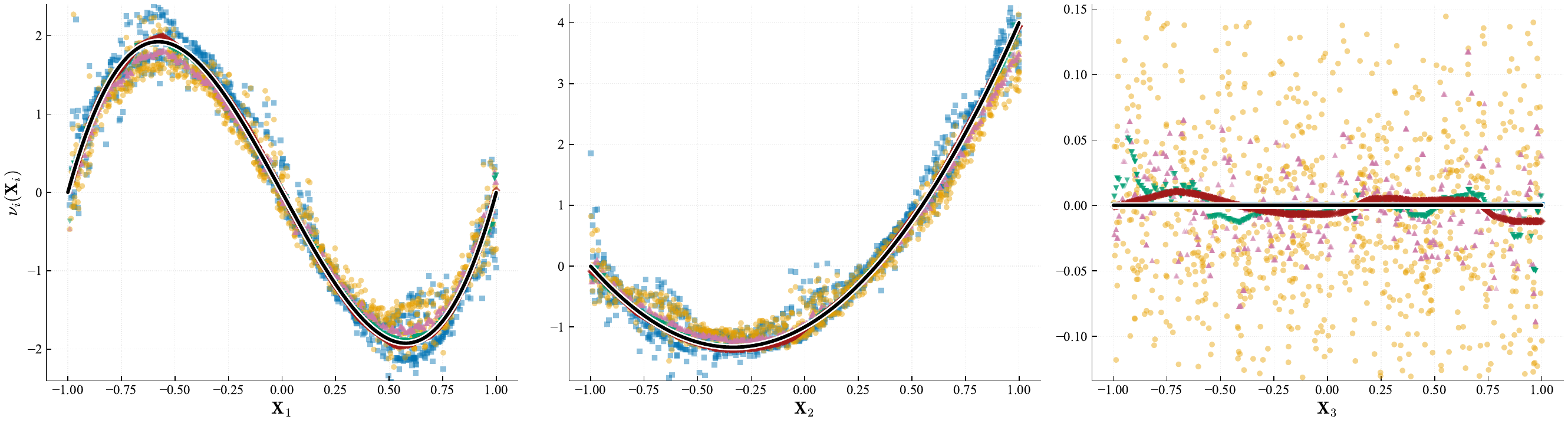}
  \caption{Estimated main effects in the analytical setting for 
${X}_1$, ${X}_2$, and the irrelevant variable ${X}_3$. 
All methods 
(\textcolor[HTML]{0072B2}{KernelSHAP}, 
\textcolor[HTML]{E69F00}{TreeSHAP}, 
\textcolor[HTML]{CC79A7}{TreeHFD}, 
\textcolor[HTML]{A31A1A}{NAM}, 
\textcolor[HTML]{009E73}{EBM}) 
closely track the \textbf{theoretical ANOVA component (in black)} on ${X}_1$ 
and ${X}_2$, and correctly assign a near-zero contribution to 
${X}_3$.}
  \label{fig:AC}
\end{figure}

\paragraph{Analytical Case.}
We consider a random vector $\mathbf{X} \in [-1,1]^3$ of following density 
\begin{equation}
    \forall \mathbf x \in [-1,1]^3, \quad p_{\rho}(\mathbf x) \coloneqq \frac{1}{8} \left( 1 + \rho ( x_1  x_2 +  x_1  x_3 + x_2 x_3 ) \right),
\end{equation}
with $\rho \in (-1/3 , 1)$ so that $p_{\rho}$ satisfies Assumption \ref{assu:densite}. The target function is defined as
\begin{equation}
    f(\mathbf X) \coloneqq f_1({X}_1) + f_2({X}_2) + f_{1,2}({X}_1, {X}_2),
\end{equation}
whose components are given in closed form in Appendix \ref{appendix:analytical}. This setting has two convenient features: the marginal and bivariate densities admit explicit expressions and the variable ${X}_3$ is irrelevant to $f$. This allows us to check if different methods correctly assign it a zero contribution. We display these theoretical components and the corresponding contributions given by the post-hoc explanation methods and interpretable models in Fig \ref{fig:AC}.

\paragraph{Real-world datasets.}
The datasets characteristics are summarized in Table~\ref{table:datasets}. We also present the fidelity reconstruction in Table~\ref{sample-table}. We illustrate in Figure \ref{fig:CI_x1} an example of visual agreement (all other figures are deferred to the Appendix).

\begin{table}[ht]
  \caption{Reconstruction fidelity of truncated functional ANOVA.}
  \label{sample-table}
  \centering
  \footnotesize
  \setlength{\tabcolsep}{3pt}
  \begin{minipage}{0.49\textwidth}
    \centering
    \begin{tabular}{llcccc}
      \toprule
      \textbf{Dataset} & \textbf{Model} & Acc. & $K$ & Rec. & Ortho. \\
      \midrule
      \multirow{4}{*}{\shortstack[l]{\textbf{BS}\\{\scriptsize(17\,379, 8)}}}
        & XGB & 0.86 & 2 & 0.92 & $9.49\mathrm{e}{-2}$ \\
        & MLP & 0.85 & 2 & 0.91 & $8.83\mathrm{e}{-2}$ \\
        & EBM & 0.83 & 2 & 0.96 & $9.52\mathrm{e}{-2}$ \\
        & NAM & --   & --& --   & -- \\
      \midrule
      \multirow{4}{*}{\shortstack[l]{\textbf{CH}\\{\scriptsize(20\,640, 8)}}}
        & XGB & 0.84 & 2 & 0.86 & $6.35\mathrm{e}{-2}$ \\
        & MLP & 0.81 & 1 & 0.88 & 0.00 \\
        & EBM & 0.83 & 2 & 0.89 & $4.04\mathrm{e}{-2}$ \\
        & NAM & 0.76 & 1 & 0.94 & 0.00 \\
      \midrule
      \multirow{4}{*}{\shortstack[l]{\textbf{CI}\\{\scriptsize(48\,842, 14)}}}
        & XGB & 0.88 & 1 & 0.91 & 0.00 \\
        & MLP & 0.85 & 1 & 0.94 & 0.00 \\
        & EBM & 0.88 & 1 & 0.90 & 0.00 \\
        & NAM & --   & --& --   & -- \\
      \bottomrule
    \end{tabular}
  \end{minipage}%
  \hfill
  \begin{minipage}{0.49\textwidth}
    \centering
    \begin{tabular}{llcccc}
      \toprule
      \textbf{Dataset} & \textbf{Model} & Acc. & $K$ & Rec. & Ortho. \\
      \midrule
      \multirow{4}{*}{\shortstack[l]{\textbf{EG}\\{\scriptsize(10\,000, 12)}}}
        & XGB & 0.93 & 2 & 0.89 & $7.41\mathrm{e}{-3}$ \\
        & MLP & 0.98 & 2 & 0.95 & $6.61\mathrm{e}{-3}$ \\
        & EBM & 0.94 & 2 & 0.99 & $5.58\mathrm{e}{-3}$ \\
        & NAM & 0.88 & 1 & 1.00 & 0.00 \\
      \midrule
      \multirow{4}{*}{\shortstack[l]{\textbf{PI}\\{\scriptsize(768, 8)}}}
        & XGB & 0.74 & 2 & 0.85 & $9.56\mathrm{e}{-2}$ \\
        & MLP & 0.73 & 2 & 0.99 & 0.00 \\
        & EBM & 0.76 & 2 & 0.97 & 0.00 \\
        & NAM & 0.73 & 1 & 1.00 & 0.00 \\
      \midrule
      \multirow{4}{*}{\shortstack[l]{\textbf{SC}\\{\scriptsize(21\,263, 81)}}}
        & XGB & 0.93 & 1 & 0.86 & 0.00 \\
        & MLP & 0.91 & 1 & 0.89 & 0.00 \\
        & EBM & 0.90 & 1 & 0.88 & 0.00 \\
        & NAM & --   & --& --   & -- \\
      \bottomrule
    \end{tabular}
  \end{minipage}
\end{table}
Briefly, \textbf{Acc.} (for \emph{Accuracy}) quantifies the predictive performance of the trained machine learning model; $K \in \{1, 2\}$ is the truncation level of our decomposition; \textbf{Rec.} (for \emph{Reconstruction}) is the reconstruction $R^2$ of our estimator, \textit{i.e.} the part of the variance from the trained machine learning model that is explained by our framework; and \textbf{Ortho.} (for \emph{Hierarchical Orthogonality}) is a cosine-based quantity measuring the \emph{worst} deviation from hierarchical orthogonality where $0$ means perfect hierarchical orthogonality. The more detailed definition of all performance metrics are given in Appendix \ref{appendix:experiments}.

\begin{figure}[t]
  \centering
  \includegraphics[width=0.48\textwidth]{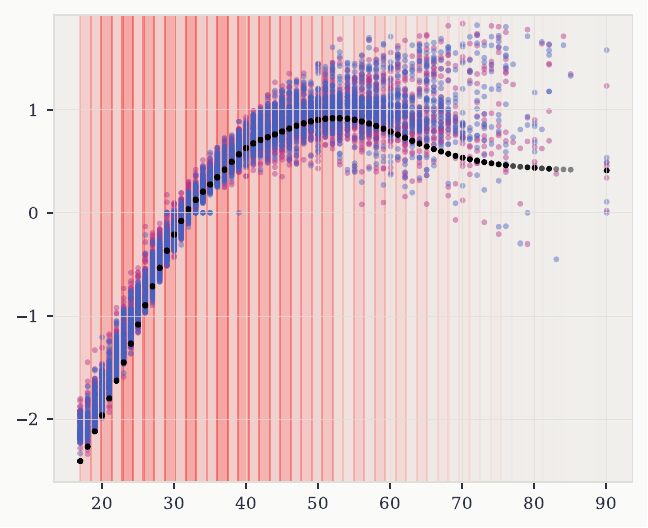}%
  \hfill
  \includegraphics[width=0.48\textwidth]{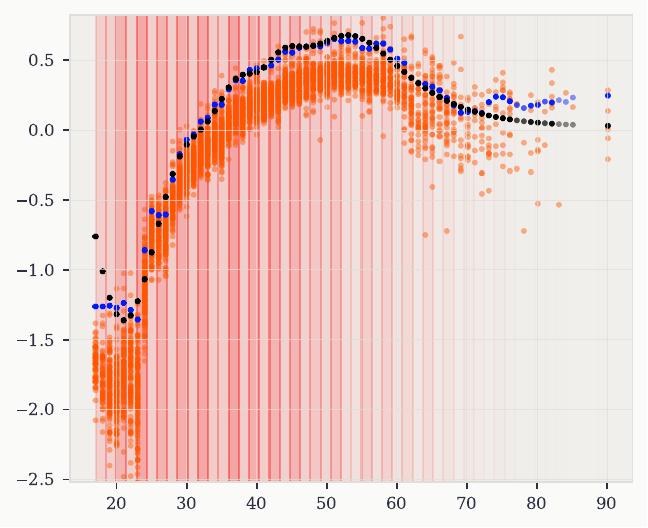}
  \caption{{Estimated first-order effect of \texttt{Age} on the \emph{Census Income} dataset.} \\
  \textbf{Left:} our functional ANOVA estimate (black) compared with \textcolor{kernelshap}{KernelSHAP} and \textcolor{deepshap}{DeepSHAP} attributions for a trained MLP. \textbf{Right:} our estimate (black) compared with the \textcolor{treehfd}{TreeHFD} main effect and \textcolor{treeshap}{TreeSHAP} attributions for a trained XGBoost model. The intensity of the \textcolor{red}{red background} represents the empirical marginal density of \texttt{Age}, with darker regions indicating higher data concentration. Our component follows the central trend of the SHAP attributions and closely agrees with TreeHFD, particularly in high-density regions.}
  \label{fig:CI_x1}
\end{figure}

\paragraph{Results.}
Across the four considered model families, truncations involving only main effects and pairwise interactions ($K\leq 2$) achieve a reconstruction $R^2$ above 0.8 on every dataset. This indicates that low-order components capture most of the variability in the fitted-model predictions in our experiments, without excluding the presence of higher-order effects. Despite the deliberately simple plug-in estimation of the marginal densities, our empirical diagnostics show only small deviations from the hierarchical-orthogonality conditions. Finally, the runtimes reported in Table~\ref{table:time} demonstrate that the resulting explanations can be computed efficiently while requiring only black-box evaluations of the model. Appendix~\ref{appendix:experiments} provides the complete experimental material, including the full collection of figures, numerical tables and detailed computation-time table.

\section{Discussion}\label{section:discussion}

\paragraph{Conclusion.}
In this work, we completely and explicitly characterized the generalized functional ANOVA components for continuous inputs on a bounded domain whose joint density is bounded above and away from zero. This representation turns the computation of the decomposition from an implicit constrained optimization problem into the estimation of coefficients in an explicit basis, thereby connecting it to classical nonparametric series estimation. As a proof of concept, we introduced a simple and model-agnostic estimator. Our experiments show that its low-order components reconstruct a large proportion of the fitted-model predictions and reveal close empirical connections with established additive explanation methods.

\paragraph{Limitations \& Future Work.}
The main limitation of this work is the absence of a statistical analysis of the estimation of the generalized functional ANOVA. Two distinct settings must be considered. First, when the input density is known, the resulting basis is oblique rather than orthogonal, requiring an adaptation of the standard nonparametric estimation literature which could notably draw on the compressed sensing literature. Second, when the density is unknown, appropriate regularity, or smoothness, classes must be introduced for both the target function and the density. One must then determine how the smoothness of the density affects the overall convergence rate when the target function has a different smoothness. A natural next step is therefore to define distribution-adapted global regularity classes, such as Sobolev-type ellipsoids expressed through the coefficients of our Riesz basis, and to construct estimators attaining minimax-optimal rates over these classes, potentially through suitably penalized least squares. Such a development would provide the statistical counterpart to our representation theorem and constitute a substantial contribution to the field of statistical learning theory.

\newpage

\subsection*{AI use statement}

In this work, we used generative AI tools for editing and English-language proofreading, as well as for assisting with literature review. In particular, we requested the LLM feedback provided by Paper Assitant Tool (PAT) available on OpenReview for the conference. The main ideas, problem formulation, contributions, theorem statements, and proofs were conceived, formulated, and written by the authors. Generative AI tools were used only subsequently to review, edit, and improve this material. We take responsibility for the final content of this work, including text, claims, or artifacts produced with the aid of generative AI.

\subsection*{Reproducibility statement}

The results presented in the main body of the paper are demonstrated in the dedicated appendix section (Appendix \ref{appendix:proofs}). Our proposed estimator is presented in Section \ref{sec:estimating}. The corresponding algorithm and detailed procedure to compute it are
provided in Appendix \ref{appendix:estimator}. Detailed descriptions of the datasets, models, and metrics used in our experiments are given in Appendix \ref{appendix:experiments}.


\bibliography{biblio_ml}
\bibliographystyle{apalike}



\newpage

\appendix

\startcontents[appendices]
\printcontents[appendices]{l}{1}{%
    \section*{Appendix}
    \setcounter{tocdepth}{2}%
}
\newpage

\section{Legendre Polynomials}\label{appendix:legendre}

In this section, we rely on standard results about Legendre polynomials (see for example \citet{szeg1939orthogonal}).

\begin{definition}[Legendre polynomial]\label{def:legendre}
For any non-negative integer $m \in \mathbb{N}$, we denote by
$\mathbf P_m$ the Legendre polynomial of degree $m$. This family is
uniquely defined by the three-term recurrence relation:
\begin{equation}
    \left\{
    \begin{aligned}
        \mathbf P_0(X) &= 1, \\
        \mathbf P_1(X) &= X, \\
        (m+1)\, \mathbf P_{m+1}(X) &= (2m+1)\, X\, \mathbf P_m(X) -m\, \mathbf P_{m-1}(X),
        \quad \forall\, m \geq 1.
    \end{aligned}
    \right.
\end{equation}
\end{definition}

\begin{proposition}
The Legendre polynomials satisfy the following identity:
\begin{equation}
 \forall n,m \in \mathbb{N}, \quad   \int_{-1}^{1} \mathbf P_n(x) \mathbf P_m(x) dx = \frac{2}{2n + 1} \mathbf{1}_{\{n = m\}}.
\end{equation}
\end{proposition}

\begin{definition}
    For any integer $m \in \mathbb N$, we denote by $\phi_m$ the \emph{normalized} Legendre polynomial of degree $m$ defined as:
    \begin{equation}
        \phi_m \coloneqq \sqrt{ \frac{2m + 1}{2} } \mathbf P_m.
    \end{equation}
\end{definition}

\begin{theorem}\cite{reed1972methods,ghanem2003stochastic,courant2024methods}
For any multi-index $\bm \alpha := ( \alpha_1, \dots, \alpha_d ) \in \mathbb{N}^d$, we define the following tensorized product of normalized Legendre polynomials:
\begin{equation}
    \Phi_{ \bm \alpha } (X_1, \dots, X_d) \coloneqq \prod\limits_{ j=1 }^{d} \phi_{\alpha_j}( X_j ).
\end{equation}
The collection $( \Phi_{\bm{\alpha}} )_{ \bm \alpha \in \mathbb N^d }$ forms a Hilbert basis (\textit{i.e.} an orthonormal basis) of $L^2( [-1,1]^d)$ equipped with the Lebesgue measure $\lambda$.
\end{theorem}

A direct consequence of the previous theorem is an \emph{ANOVA-type indexing}
of the Hilbert basis of $L^2(  \lambda )$, in which each basis
element depends only on a prescribed subset of variables. Recall that
$\mathbb{N}_+$ denotes the set of positive integers. For any $S \subseteq [d]$ and any $\bm m_S \coloneqq (m_j)_{j \in S} \in \mathbb{N}_+^{|S|}$, set
\begin{equation}
    \psi_S^{(\bm m_S)} \coloneqq \frac{1}{\sqrt{2^{d - \vert S \vert}}} \prod_{j \in S} \phi_{m_j},
\end{equation}
with the convention $\psi_{\emptyset} \coloneqq 1/\sqrt{2^d}$. Then the family
\begin{equation}
    \left\{ \, \psi_S^{(\bm m_S)} \;:\; S \subseteq [d], \ \bm m_S \in \mathbb{N}_+^{|S|} \, \right\}
\end{equation}
is a Hilbert basis of $L^2(  \lambda )$. Indeed, for any
$\bm\alpha \in \mathbb{N}^d$, let
$S_{\bm\alpha} \coloneqq \{ j \in [d] : \alpha_j > 0 \}$ denote its support
and $\bm\alpha_{S_{\bm\alpha}} \in \mathbb{N}_+^{|S_{\bm\alpha}|}$ its
restriction to the nonzero coordinates. Since $\phi_0 = 1/\sqrt{2}$, one has
\begin{equation}
    \Phi_{\bm\alpha} = \prod_{j=1}^d \phi_{\alpha_j}
    = \frac{1}{\sqrt{2^{ d - \vert S_{\bm \alpha} \vert }}} \prod_{j \in S_{\bm\alpha}} \phi_{\alpha_j}
    = \psi_{S_{\bm\alpha}}^{(\bm\alpha_{S_{\bm\alpha}})}.
\end{equation}
\begin{remark}
    A key consequence of the orthogonality of the basis is the orthogonality to the constant:
    \begin{equation}
        \forall S \subseteq [d] \setminus \emptyset, \: \forall \bm{m}_S \in \mathbb{N}^{\vert S \vert}_+, \quad \int_{[-1,1]^d} \psi_{S}^{(\bm m_{S})} ( \mathbf x ) d \lambda( \mathbf x ) = 0.
    \end{equation}
\end{remark}

\newpage
\section{Proofs}\label{appendix:proofs}

\subsection{Proof of Theorem \ref{theorem:riesz}}

\begin{lemma}\label{lemma:equiv_norm}
The space $L^2( [-1,1]^d , \lambda )$, that we will just denote $L^2(\lambda)$ for simplicity, is equivalent in norm to the space $L^2(P_{\mathbf X})$.
\end{lemma}

\begin{proof}
Let $g : [-1,1]^d \to \mathbb{R}$ be a $P_{\mathbf X}-$square integrable function. By definition of $P_{\mathbf X}$ one has:
\begin{equation}
    \|g\|_{L^2(P_{\mathbf X})}^2 \;=\; \int_{[-1,1]^d} g(\mathbf x)^2 p(\mathbf x) \, d\lambda(\mathbf x).
\end{equation}
Using that $c_1 \leq p \leq c_2$, we obtain:
\[
c_1 \int_{[-1,1]^d} g(\mathbf x)^2 \, d\lambda(\mathbf x) \;\leq\; \int_{[-1,1]^d} g(\mathbf x)^2 p(\mathbf x) \, d\lambda(\mathbf x) \;\leq\; c_2 \int_{[-1,1]^d} g(\mathbf x)^2 \, d\lambda(\mathbf x),
\]
which finally gives the norm equivalence.
\end{proof}

\begin{definition}\label{def:operator}
    For every set $S \subseteq [d]$, let $\Pi_S$ be the orthogonal projection onto
    \begin{equation}
        V_S \coloneqq \overline{\mathrm{span}}\left\{ \psi_S^{(\bm m_S)} \;\middle|\; \bm m_S \in \mathbb{N}_+^{\lvert S \rvert} \right\},
    \end{equation}
    where $\overline{\mathrm{span}}$ denotes the closure of the linear span in $L^2(\lambda)$.
    Let also $\{w_S\}_{S \subseteq [d]}$ be a given collection of measurable functions, satisfying for all $S \subseteq [d]$ and all $\mathbf{x}_S \in [-1,1]^{\vert S \vert}$:
    \begin{equation}
        0 < \kappa_1 \leq w_S(\mathbf{x}_S) \leq \kappa_2 < \infty,
    \end{equation}
    for a given couple of positive constants $0 < \kappa_1 \leq \kappa_2 < \infty$.
    We define the operator $\mathcal{T}$ as follows:
    \begin{equation}
        \mathcal{T} : \left\{
        \begin{array}{rcl}
            L^2( \lambda ) & \longrightarrow & L^2( \lambda ) \\[4pt]
            u & \longmapsto & \displaystyle\sum_{S \subseteq [d]} w_S \, \Pi_S(u),
        \end{array}
        \right.
    \end{equation}
    Note that $\mathcal T$ depends on the given collection of functions $\{w_S\}_{S \subseteq [d]}$ but we simply denote it by $\mathcal T$ for simplicity.
\end{definition}

\begin{proposition}\label{prop:bounded}
    The operator $\mathcal T$ is bounded.
\end{proposition}

\begin{proof}
    Using that $\bigl( \psi_S^{(\bm m_S)} \bigr)_{S \subseteq [d],\, \bm m_S \in \mathbb{N}_+^{\lvert S \rvert}}$ is a Hilbert basis of $L^2( \lambda )$, we have the orthogonal decomposition
\begin{equation}
    L^2 \left(\lambda\right) = \bigoplus_{S \subseteq [d]}^{\perp} V_S.
\end{equation}
Consequently, any $u \in L^2( \lambda )$ admits the unique orthogonal decomposition
\begin{equation}
    u = \sum_{S \subseteq [d]} u_S, \qquad u_S \coloneqq \Pi_S(u).
\end{equation}
Let $u \in L^2( \lambda ) $, we have:
\begin{align}
    \| \mathcal{T}(u) \|_{L^2( \lambda )} &= \left \| \sum\limits_{S \subseteq [d]} w_S u_S
    \right \| \\
    &\leq \sum\limits_{ S \subseteq [d] } \| w_S u_S \| \\
    &\leq \kappa_2 \sum\limits_{ S \subseteq [d] } \| u_S \| \\
    &\leq \kappa_2 2^{ d/2 } \| u \|
\end{align}
which gives the desired result.
\end{proof}

\begin{proposition}\label{prop:invertible}
    The operator $\mathcal T$ is invertible.
\end{proposition}

\begin{proof}
    For every $S , S' \subseteq [d]$, we define the following operator:
    \begin{equation}
        A_{S , S'} : \left\{
        \begin{array}{rcl}
            V_S & \longrightarrow & V_{S'} \\[4pt]
            u_S & \longmapsto & \displaystyle (\Pi_{S'} \circ \mathcal{T}) (u_S)
        \end{array}
        \right.
    \end{equation}
    We have the following identity for every $u \in L^2(  \lambda )$:
    \begin{equation}
        \mathcal T(u) = \sum\limits_{ S \subseteq [d] } \sum\limits_{ S' \subseteq [d] } A_{S , S'}( \Pi_S(u) ).
    \end{equation}
   We denote by $L^2_S(\lambda)$ the Hilbert space of square-integrable functions depending only on the variables indexed by $S \subseteq [d]$. Observe that for every $S \subseteq [d]$, the operator $\mathcal{T}$ leaves this space invariant, i.e.,
\begin{equation}
    \mathcal{T}\left( L^2_S(\lambda) \right) 
    \subseteq L^2_S(\lambda).
\end{equation}
Indeed, let $S \subseteq [d]$, the space $ L^2_S(\lambda) $ admits the following orthogonal decomposition:
\begin{equation}
    L^2_S(\lambda) = \bigoplus_{T \subseteq S}^{\perp} V_T.
\end{equation}
So for any $u \in L^2_S(\lambda)$, $u$ uniquely expresses as $u = \sum_{T \subseteq S} \Pi_T(u)$. By applying $\mathcal T$ to this sum, we obtain:
\begin{equation}
    \mathcal{T}(u) = \sum_{T \subseteq S} w_T \Pi_T(u),
\end{equation}
which is obviously an element of $L^2_S(\lambda)$.
Furthermore, observe that we have the following decomposition:
\begin{equation}
    L^2 \left(\lambda\right) = \underbrace{\bigoplus_{T \subseteq S}^{\perp} V_T}_{ L^2_S(\lambda) } \; \bigoplus \; \bigoplus_{S' \not\subseteq S}^{\perp} V_{S'}
\end{equation}
    A direct consequence of this stabilization property using the previous decomposition is:
    \begin{equation}
        \forall S,S' \subseteq [d], \quad S' \not \subseteq S \implies A_{S,S'} = 0,
    \end{equation}
    Consequently, the operator matrix $\mathbf A := ( A_{S,S'} )_{S , S' \subseteq [d]}$ is triangular (an illustration in dimension 3 is provided in \eqref{eq:triangle}). To show the invertibility of $\mathcal T$, it suffices to prove that all the diagonal blocs are invertible \cite{tretter2008spectral}, indeed in any finite partially-ordered family of blocks where off-diagonal blocks vanish outside $S' \subseteq S$, invertibility follows by induction on $\vert S \vert$. Let $S \subseteq [d]$, we will show that $A_{S,S}$ is invertible. First, $A_{S,S}$ is clearly a linear operator and so an endomorphism of $V_S$. Let take $u_S \in V_S$, we have:
    \begin{align}
        \left\langle A_{S,S}(u_S) , u_S \right\rangle &= \left\langle (\Pi_{S} \circ \mathcal{T}) (u_S) , u_S \right\rangle \\
        &= \left\langle \Pi_{S} ( w_S u_S ) , u_S \right\rangle \\
        &= \left\langle w_S u_S , \Pi_{S}^{\star} (u_S) \right\rangle \\
        &= \left\langle w_S u_S , \Pi_{S} (u_S) \right\rangle \\
        &= \left\langle w_S u_S , u_S \right\rangle \\
        &= \int w_S {u_S}^2 d\lambda \\
        &\geq \kappa_1 \| u_S \|_{ L^2( \lambda ) }^2,
    \end{align}
    so the operator $A_{S,S}$ is \emph{coercive}. A similar computation would show that the operator is \emph{bounded}, \textit{i.e.}
    \begin{equation}
        \forall u_S , v_S \in V_S, \quad \vert \left\langle A_{S,S}(u_S) , v_S \right\rangle \vert \leq \kappa_2 \| u_S \|_{ L^2( \lambda ) } \| v_S \|_{ L^2( \lambda ) }.
    \end{equation}
    Finally, by applying the Lax-Milgram theorem \citep{brezis2011functional}, we obtain that $A_{S,S}$ is invertible and the desired result.
\end{proof}

\begin{remark}
    Let $d = 3$. We illustrate the triangular structure of the matrix $\mathbf{A}$ and of the operator $\mathcal{T}$. Order the subsets of $[d]$ as follows:
    \begin{equation}
        \left\{ \emptyset,\ \{1\},\ \{2\},\ \{3\},\ \{1,2\},\ \{1,3\},\ \{2,3\},\ \{1,2,3\} \right\}.
    \end{equation}
    The coefficient $A_{S,S'}(u_S)$ denotes the projection onto $V_{S'}$ of $w_S u_S$, which belongs to $L^2_S(\lambda)$. In the case $d = 3$, the matrix takes the form
\begin{equation}\label{eq:triangle}
\renewcommand{\arraystretch}{1.25}
\setlength{\arraycolsep}{6pt}
\begin{blockarray}{ccccccccc}
        & \emptyset & \{1\} & \{2\} & \{3\} & \{1,2\} & \{1,3\} & \{2,3\} & \{1,2,3\} \\
\begin{block}{c[cccccccc]}
\emptyset   & \ast &         &         &         &         &         &         &         \\
\{1\}       & \ast & \ast &         &         &         &         &         &         \\
\{2\}       & \ast &         & \ast &         &         &         &         &         \\
\{3\}       & \ast &         &         & \ast &         &         &         &         \\
\{1,2\}     & \ast & \ast & \ast &         & \ast &         &         &         \\
\{1,3\}     & \ast & \ast &         & \ast &         & \ast &         &         \\
\{2,3\}     & \ast &         & \ast & \ast &         &         & \ast &         \\
\{1,2,3\}   & \ast & \ast & \ast & \ast & \ast & \ast & \ast & \ast \\
\end{block}
\end{blockarray}
\end{equation}
where $\ast$ indicates entries that are potentially non-zero.
\end{remark}

\begin{proposition}\label{prop:hom}
    $\mathcal T$ is a linear homeomorphism of $L^2(  \lambda )$.
\end{proposition}

\begin{proof}
    As a direct consequence of proposition \ref{prop:bounded} and proposition \ref{prop:invertible}, $\mathcal T$ is a bounded automorphism of $L^2(  \lambda )$. Using the Banach-Schauder theorem \citep{brezis2011functional}, $\mathcal T$ is a homeomorphism.
\end{proof}

\begin{proof}[Proof of Theorem \ref{theorem:riesz}]

First, observe that since $c_1 \leq p \leq c_2$, we have for all sets $S \subseteq [d]$:
\begin{equation}
    c_1 2^{ d - \vert S \vert } \leq p_S \leq c_2 2^{ d - \vert S \vert }.
\end{equation}
Consequently, a uniform bound for all $p_S$ is given by:
\begin{equation}
    c_1 \leq p_S \leq c_2 2^d.
\end{equation}

Moreover, recall that our proposed collection of functions $\Xi$ is given by:
\begin{equation}
        \forall \mathbf x \in [-1,1]^d, \quad \Xi( \mathbf x ) = \left( \left( \frac{ \psi_S^{( \bm m_S )} (\mathbf x) }{ p_S( \mathbf x_S ) } \right)_{ \bm{m}_S \in \mathbb{N}_+^{\vert S \vert} } \right)_{ S \subseteq [d] }.
    \end{equation}
So we can rewrite:
\begin{equation}
    \forall S \subseteq [d], \forall \bm m_S \in \mathbb{N}_+^{ \vert S \vert }, \quad \xi_S^{ ( \bm m _S ) } = \mathcal{T}( \psi_S^{( \bm m_S )} ),
\end{equation}
by defining the weight functions $w_S$ as follows: $ w_S \coloneqq 1 / p_S $ with $\kappa_1 \coloneqq 1/(c_2 2^d)$ and $\kappa_2 \coloneqq 1/c_1$.
Thanks to Proposition \ref{prop:hom}, $\Xi$ is the direct image of a Hilbert basis by a homeomorphism $\mathcal T$ which implies that $\Xi$ is a \emph{Riesz basis} of $L^2( \lambda )$. In particular, by Lemma \ref{lemma:equiv_norm} the family $\Xi$ is also a Riesz basis of $L^2( P_{\mathbf X} )$ which gives the desired conclusion.
\end{proof}

\subsection{Proof of Lemma \ref{lemma:hierarchical}}

\begin{proof}
Let us consider $T \subsetneq S \subseteq [d]$, $\bm m_S \coloneqq (m_s)_{s \in S} \in \mathbb{N}_+^{ \vert S \vert }$, $\bm n_T \coloneqq (n_t)_{t \in T} \in \mathbb{N}_+^{ \vert T \vert }$. For simplicity, we define $C := S \setminus T$ and consider $ \bm m_C \coloneqq (m_c)_{c \in C} $. We have:
\begin{align}
    \mathbb{E}\left[ \xi_S^{(\bm{m}_S)}( \mathbf X ) \cdot \xi_T^{(\bm{n}_T)}( \mathbf X ) \right] &= \int_{[-1,1]^d} \xi_S^{(\bm{m}_S)}( \mathbf x ) \cdot \xi_T^{(\bm{n}_T)}( \mathbf x ) \cdot p( \mathbf x ) d \lambda( \mathbf x ), \\
    &= \int_{[-1,1]^{ \vert S \vert }} \xi_S^{(\bm{m}_S)}( \mathbf x ) \cdot \xi_T^{(\bm{n}_T)}( \mathbf x ) \cdot p_S( \mathbf x_S ) d\lambda( \mathbf x_S ), \\
    &\propto \int_{[-1,1]^{ \vert S \vert }} \frac{ \prod\limits_{s \in S} \phi_{m_s}(\mathbf x_s) }{\textcolor{red}{p_S( \mathbf x_S )}} \frac{ \prod\limits_{t \in T} \phi_{n_t}(\mathbf x_t) }{p_T( \mathbf x_T )} \textcolor{red}{p_S( \mathbf x_S )} d\lambda( \mathbf x_S ), \\
    &\propto \int_{[-1,1]^{ \vert S \vert }} \prod\limits_{s \in S} \phi_{m_s}(\mathbf x_s) \frac{ \prod\limits_{t \in T} \phi_{n_t}(\mathbf x_t) }{p_T( \mathbf x_T )} d\lambda( \mathbf x_S ), \\
    &\propto \int_{[-1,1]^{ \vert S \vert }} \textcolor{blue}{\underbrace{\prod\limits_{c \in C} \phi_{m_c}(\mathbf x_c)}_{ \sqrt{2^{d - \vert C \vert}} \psi_{C}^{( \bm m_{C} )}( \mathbf x_{ C } ) }} \cdot \textcolor{orange}{ \underbrace{\frac{\prod\limits_{t \in T} \phi_{m_t}(\mathbf x_t) \phi_{n_t}(\mathbf x_t) }{p_T( \mathbf x_T )}}_{u( \mathbf x_T )}} d\lambda( \mathbf x_S ), \\
    &\propto \textcolor{blue}{\underbrace{\int_{ [-1,1]^{ \vert C \vert } } \psi_{C}^{( \bm m_{C} )}( \mathbf x_{ C } ) d\lambda( \mathbf x_{C} )}_{ = 0}} \cdot \textcolor{orange}{\int_{[-1,1]^{\vert T \vert}} u( \mathbf x_T ) d \lambda( \mathbf x_T )}, \\
    &= 0.
\end{align}
\end{proof}

\begin{remark}
    In particular, all basis elements are centered, \textit{i.e.}
\begin{equation}
    \forall S \subseteq [d] \setminus \emptyset, \quad \mathbb{E}\left[ \xi_S^{(\bm{m}_S)}( \mathbf X ) \right] = 0.
\end{equation}
\end{remark}

\subsection{Proof of Theorem \ref{thm:representation}}

\begin{proof}[Proof of Theorem \ref{thm:representation}]
By Theorem \ref{theorem:riesz}, every function
$f \in L^2(P_{\mathbf X})$ admits a unique representation
\begin{equation}
    f(\mathbf X)
    =
    \sum_{S \subseteq [d]}
    \sum_{\bm m_S \in \mathbb N_+^{|S|}}
    \theta_S^{(\bm m_S)}
    \xi_S^{(\bm m_S)}(\mathbf X),
\end{equation}
where the coefficient sequence
\[
\bm{\theta}
\coloneqq
\bigl(
\theta_S^{(\bm m_S)}
\bigr)_{S \subseteq [d],\,
          \bm m_S \in \mathbb N_+^{|S|}}
\]
is uniquely determined by $f$ and belongs to $\ell^2$.

For every $S \subseteq [d]$, define
\begin{equation}
    \mathcal H_S
    \coloneqq
    \overline{\operatorname{span}}
    \left\{
        \xi_S^{(\bm m_S)}
        \;\middle|\;
        \bm m_S \in \mathbb N_+^{|S|}
    \right\}.
\end{equation}
Since $\Xi$ is a Riesz basis, its expansion converges
unconditionally. Consequently, the series can be regrouped
according to the subsets $S$, and
\begin{equation}
    L^2(P_{\mathbf X})
    =
    \bigoplus_{S \subseteq [d]} \mathcal H_S,
\end{equation}
where the right-hand side denotes an algebraic direct sum,
which is not orthogonal in general.

Moreover, if $T \subsetneq S$, Lemma
\ref{lemma:hierarchical} shows that every basis element of
$\mathcal H_S$ is orthogonal to every basis element of
$\mathcal H_T$. By linearity and continuity of the inner
product, this extends to the closures, so that
\begin{equation}
    T \subsetneq S
    \quad\Longrightarrow\quad
    \mathcal H_T \perp \mathcal H_S.
\end{equation}

For every $T \subseteq [d]$, the same Riesz-basis
decomposition applied to functions of $\mathbf X_T$ yields
\begin{equation}
    L_T^2
    =
    \bigoplus_{U \subseteq T} \mathcal H_U.
\end{equation}
Hence, if $T \subsetneq S$, then every $U \subseteq T$ is a
strict subset of $S$, and therefore
\begin{equation}
    \mathcal H_S \perp L_T^2.
\end{equation}

Finally, define
\begin{equation}
    f_S
    \coloneqq
    \sum_{\bm m_S \in \mathbb N_+^{|S|}}
    \theta_S^{(\bm m_S)}
    \xi_S^{(\bm m_S)}.
\end{equation}
By definition, we have $f=\sum_{S\subseteq[d]}f_S$, while the previous argument shows
that
\[
    \langle f_S,g_T\rangle_{L^2(P_{\mathbf X})}=0
    \qquad
    \text{for every } T\subsetneq S
    \text{ and every } g_T\in L_T^2.
\]
Thus the family $(f_S)_{S\subseteq[d]}$ satisfies the ANOVA decomposition \eqref{eq:fanova_objective} and the hierarchical orthogonality constraint \eqref{eq:fanova_orthogonality} which completes the proof.
\end{proof}

\subsection{Proof of Corollary \ref{cor:independent}}

\begin{proof}
    To show that the resulting decomposition is mutually orthogonal by set, it suffices to prove that for any sets $S,T$ such that $S \neq T$ and for any vector of indices $\bm m_S , \bm n_T$, one has:
    \begin{equation}
        \mathbb E \left[ \xi_S^{ ( \bm m_S ) }( \mathbf X) \cdot \xi_T^{ ( \bm n_T ) }( \mathbf X) \right] = 0.
    \end{equation}
    Under independence, the density expresses as the following product $p = p^{\otimes d} \coloneqq p_1 \times \dots \times p_d$. The elements of our basis becomes:
    \begin{equation}
        \forall S \subseteq [d], \forall \bm m_S \in \mathbb N_+^{ \vert S \vert }, \quad \xi_S^{(\bm m_S)} = \frac{1}{ \sqrt{ 2^{ d - \vert S \vert } } } \cdot \prod\limits_{j \in S} \frac{ \phi_{m_j} }{p_j}.
    \end{equation}
    Let $S \neq T \subseteq [d]$, without loss of generality, we fix $s^{\star} \in S$ such that $s^{\star} \notin T$. We have:
    \begin{align}
        \mathbb E \left[ \xi_S^{ ( \bm m_S ) }( \mathbf X) \cdot \xi_T^{ ( \bm n_T ) }( \mathbf X) \right] &= \mathbb E \left[ \frac{1}{ \sqrt{ 2^{ d - \vert S \vert } } } \prod\limits_{s \in S} \frac{ \phi_{m_s}(  X_s ) }{p_s (  X_s )} \cdot \frac{1}{ \sqrt{ 2^{ d - \vert T \vert } } } \prod\limits_{t \in T} \frac{ \phi_{n_t} (  X_t ) }{p_t (  X_t )} \right] \\
        &\propto \mathbb E \left[ \prod\limits_{s \in S} \frac{ \phi_{m_s}(  X_s ) }{p_s (  X_s )} \cdot \prod\limits_{t \in T} \frac{ \phi_{n_t} (  X_t ) }{p_t (  X_t )} \right] \\
        &\propto\mathbb{E}\left[ \textcolor{blue}{ \underbrace{\frac{ \phi_{m_{s^{\star}}}(  X_{s^{\star}} ) }{p_{s^{\star}} (  X_{s^{\star}} )} }_{ u(  X_{s^{\star}} ) }} \cdot \textcolor{orange}{ \underbrace{\prod\limits_{s \in S \setminus\{s^{\star}\}} \frac{ \phi_{m_s}(  X_s ) }{p_s (  X_s )} \cdot \prod\limits_{t \in T} \frac{ \phi_{n_t} (  X_t ) }{p_t (  X_t )}}_{ v(  X_{ [d] \setminus \{ s^{\star} \} } ) } } \right] \\
        &\propto \textcolor{blue}{ \mathbb{E}\left[ \frac{ \phi_{ m_{ s^{\star} } }(  X_{ s^{\star} } ) }{ p_{s^{\star}}(  X_{ s^{\star} } ) } \right] } \cdot \textcolor{orange}{ \mathbb{E}\left[ \prod\limits_{s \in S \setminus\{s^{\star}\}} \frac{ \phi_{m_s}(  X_s ) }{p_s (  X_s )} \cdot \prod\limits_{t \in T} \frac{ \phi_{n_t} (  X_t ) }{p_t (  X_t )} \right] } \\
        &\propto \mathbb{E}\left[ \frac{ \phi_{ m_{ s^{\star} } }(  X_{ s^{\star} } ) }{ p_{s^{\star}}(  X_{ s^{\star} } ) } \right] \\
        &\propto \int_{-1}^{1} \frac{ \phi_{ m_{ s^{\star} } }( x ) }{ \textcolor{red}{p_{s^{\star}}( x ) }} \textcolor{red}{p_{s^{\star}}( x )} dx \\
        &\propto \int_{-1}^{1} \phi_{ m_{ s^{\star} } }( x ) dx \\
        &= 0.
    \end{align}
\end{proof}

\newpage
\section{Estimation}\label{appendix:estimator}

In this section, we detail all the computations behind our method to estimate the functional ANOVA from a mere data sample. Throughout this section, we denote by $\mathcal X := \{ \mathbf X^{(i)} \}_{i=1}^{n}$ and $ \mathcal Y := \{ f( \mathbf X^{(i)} ) \}_{i=1}^{n}$ the data, and $\{ \mathcal X , \mathcal Y \}$ is the corresponding data sample. If in the vast majority of datasets, one can assume that the features are bounded, they rarely live natively in $[-1,1]^d$. We illustrate in Figure \ref{fig:global-scheme} the entire procedure to estimating the truncated functional decomposition from data.

\begin{figure}[ht]
  \centering
  \begin{minipage}[c]{0.50\textwidth}
    \centering
    \begin{tikzpicture}[scale=0.6, transform shape,
      font=\sffamily\small,
      node distance=6mm and 10mm,
      >={Stealth[length=2.4mm, width=2mm]},
      every edge/.style={draw=black, line width=0.5pt, ->},
      block/.style={
        rectangle, rounded corners=4pt,
        draw=black, line width=0.5pt,
        fill=blue!15,
        minimum width=44mm, minimum height=10mm,
        align=center, inner sep=3pt
      },
      side/.style ={block, minimum height=14mm},
      final/.style={block, minimum height=10mm},
      orangebox/.style={fill=orange!25},
      violetbox/.style={fill=violet!20},
      junc/.style ={circle, draw=black, fill=white, line width=0.5pt,
                    inner sep=0pt, minimum size=4.4mm, font=\scriptsize}
    ]
      \newcommand{\cplx}[1]{%
        \\[1pt]{\scriptsize\color{black!60}#1}}
      \node[block] (h1)
        {Dataset\\[1pt]\footnotesize $\{\mathcal{X},\,\mathcal{Y}\}$};
      \node[block, below=of h1] (h2)
        {Feature scaling\\[1pt]
         \footnotesize $\mathrm{StandardScaling} \;+\; \tanh$};
      \node[side, orangebox, below left=8mm and -10mm of h2]  (g3)
        {Denominator\\[1pt]
         \footnotesize Density estimation
         \cplx{$\mathcal{O}\!\left(n\times d^{K}\right)$}};
      \node[side, below right=8mm and -10mm of h2] (d3)
        {Numerator\\[1pt]
         \footnotesize Polynomial evaluation
         \cplx{$\mathcal{O}\!\left(n\times (d * M)^{K}\!/K!\right)$}};
      \node[final, below=28mm of h2] (b4)
        {Design matrix
         \cplx{$\mathcal{O}\!\left(n\times (d * M)^{K}\!/K!\right)$}};
      \node[final, orangebox, below=of b4] (b5)
        {Linear solving\\[1pt]
         \footnotesize Model selection (LARS) \\ \footnotesize Reduced solving (SVD)};
      \node[final, violetbox, below=of b5] (b6)
        {Estimator $\widehat{f}$};
      \coordinate (fork)  at ($(h2.south)+(0,-4mm)$);
      \node[junc] (merge) at (b4.north |- g3.center) {$\oplus$};
      \draw[->] (h1) -- (h2);
      \draw     (h2.south) -- (fork);
      \draw[->, shorten >=1pt] (fork) -| (g3.north);
      \draw[->, shorten >=1pt] (fork) -| (d3.north);
      \draw[->, shorten >=1pt] (g3.east) -- (merge);
      \draw[->, shorten >=1pt] (d3.west) -- (merge);
      \draw[->] (merge) -- (b4.north);
      \draw[->] (b4) -- (b5);
      \draw[->] (b5) -- (b6);
    \end{tikzpicture}
    \captionof{figure}{Illustration of our method to estimate the decomposition \emph{from scratch} (see Appendix \ref{appendix:estimator}).}
    \label{fig:global-scheme}
  \end{minipage}\hfill
\end{figure}

\paragraph{Feature scaling.} The first step to apply our framework on a real world dataset is a feature scaling in $[-1,1]$. We propose to address this by first centering and reducing our data and then applying $\tanh$ to each feature. This can be described as follows:
\begin{equation}
\begin{tikzpicture}[
  baseline=(current bounding box.center),
  font=\small,
  >={Stealth[length=2mm, width=1.6mm]},
  node distance=8mm,
  box/.style={rectangle, rounded corners=3pt, draw=black, line width=0.4pt,
              fill=blue!10, inner sep=4pt,
              minimum height=8mm, minimum width=14mm},
  arrlbl/.style={font=\scriptsize\sffamily, midway, above=2pt}
]
  \node[box] (X)                  {$\mathcal{X}$};
  \node[box, right=25mm of X]     (X01) {$\mathcal{X}_{(0,1)}$};
  \node[box, right=25mm of X01]   (Xs)  {$\mathcal{X}_{\mathrm{scaled}}$};

  \draw[->] (X)   -- (X01)
    node[arrlbl] {$\mathrm{StandardScaling}$};
  \draw[->] (X01) -- (Xs)
    node[arrlbl] {$\tanh$};
\end{tikzpicture}
\end{equation}

\begin{tcolorbox}[red_style]
    Features are standardized to zero mean and unit variance, using the \texttt{StandardScaler} from \texttt{scikit-learn} \cite{pedregosa2011scikit}.
\end{tcolorbox}

\paragraph{Density estimation.} We follow the standard projection procedure to estimate a density function using an orthonormal basis, see for example \citet{cencov1962estimation} and \citet{tsybakov2009introduction}. From now on, we assume that the feature sample takes values in $[-1,1]^d$. Our theoretical guarantees depend on the true marginal densities $p_S$ for $S \subseteq [d]$, which are unavailable in practice and must be estimated from the data. We propose to also estimate each $p_S$ using a truncated tensorized Legendre expansion. In practice, for any $S \subseteq [d]$ and a truncation degree $d_{\text{density}}$, we approximate $p_S$ as follows:
\begin{equation}
    p_S \;\approx\; \sum_{\bm{m} \in \{0,\dots,d_{\text{density}}\}^{|S|}} c_{S}^{(\bm{m})} \prod\limits_{j \in S} \phi_{m_j},
\end{equation}
By orthonormality of the tensorized basis, each coefficient $c_{S}^{(\bm{m})}$ is estimated as follows:
\begin{equation}
    \widehat{c}_{S,n}^{(\bm{m})} \;:=\; \frac{1}{n}\sum_{i=1}^{n} \prod\limits_{j \in S} \phi_{m_j}( X_j^{(i)} ).
\end{equation}
This choice is motivated by two considerations. First, it keeps the whole pipeline fully transparent and free of additional black-box components: each coefficient reduces to an empirical average, the estimator is trained from scratch with a single hyperparameter (the truncation degree $d_{\text{density}}$), and it integrates seamlessly with the rest of our analysis, which is already expressed in the Legendre basis. Second, and more importantly, our goal is not to recover $p_S$ pointwise but to enforce the hierarchical orthogonality constraint~\eqref{eq:fanova_orthogonality}; since capturing a \emph{global trend} of $p_S$ is enough to guarantee this condition, the pointwise quality of the density estimate is largely irrelevant to our objective.
\begin{remark}
In theory the number of densities to be estimated scales as $\mathcal{O}(d^K)$, but in practice we restrict ourselves to $K \in \{1,2\}$, which keeps the estimation step tractable even in moderately high dimension.
\end{remark}
\begin{remark}[Density clipping]
In practice, we replace the raw estimate $\hat p_S$ by $\max(\varepsilon, \hat p_S)$ (typically $\varepsilon = 10^{-2}$), in order to avoid numerical instabilities caused by regions where the estimated density is very close to zero (or negative, since the projection estimator is not constrained to be nonnegative).
\end{remark}

\begin{tcolorbox}[red_style]
    Legendre polynomials are evaluated using \texttt{eval\_legendre} from \texttt{SciPy} \cite{virtanen2020scipy}. The truncation degree $d_{\text{density}}$ and the clipping parameter $\varepsilon$ constitute two hyperparameters of the method, referred to as \texttt{deg\_density} and \texttt{density\_clip} respectively.
\end{tcolorbox}

\paragraph{Design matrix.} With a slight abuse of notation, we still write $p_S$ for the clipped estimated densities. We then form the \emph{design matrix} associated with our basis functions,
\begin{equation}
    \mathbf{B} \;:=\; \left( \xi_S^{(\bm{m}_S)}(\mathbf{X}^{(i)}) \right)_{\substack{i \in \{1,\dots,n\} \\ (S, \bm{m}_S) \in \mathcal{I}_K^{M}}} \;\in\; \mathbb{R}^{n \times \mathcal N(d,K,M)}.
\end{equation}

\begin{tcolorbox}[red_style]
The numerator of the design matrix is similarly computed via Legendre polynomials
using \texttt{eval\_legendre} from \texttt{SciPy}. The resulting matrix
$\mathbf{B}$ is stored as a two-dimensional \texttt{NumPy} array.
\end{tcolorbox}

\paragraph{Linear solving.} The estimator defined in~\eqref{eq:estimator_a} can now be expressed as follows:
\begin{equation}
    \widehat{\bm \theta}_n^{K,M} \;\in\; \operatornamewithlimits{argmin}_{\bm \beta \in \mathbb{R}^{\mathcal N(d,K,M)}} \left\| \mathbf{B}\, \bm \beta \;-\; \bm{y} \right\|_2^2, \qquad \bm{y} \;:=\; \left(f(\mathbf{X}^{(i)})\right)_{1 \leq i \leq n} \in \mathbb{R}^{n},
\end{equation}
where $ \| \cdot \|_2 $ denotes the Euclidean norm on $\mathbb{R}^n$. We proceed in two stages: we first select a relevant subset of columns via a LARS procedure~\cite{Efron_2004,blatman2011adaptive}, and then solve the resulting reduced system using an SVD.

\begin{tcolorbox}[red_style]
Model selection is carried out using \texttt{LassoLarsIC} from \texttt{scikit-learn}, which traverses the LARS regularization path and selects the optimal support via the Bayesian Information Criterion (BIC). This yields a subset of column indices defining the reduced design matrix $\mathbf{B}_{\text{red}}$. The corresponding least-squares system is then solved via \texttt{numpy.linalg.lstsq} \cite{harris2020array}, which computes the minimum-norm solution through a thin SVD decomposition.
\end{tcolorbox}

\paragraph{Estimator.} We denote by $\widehat{f}$ the resulting estimator, with
$\widehat{f}_S$ its component associated to each subset $S \subseteq [d]$. The
full decomposition over the dataset denoted \texttt{X,y} is obtained by calling a function of the following arguments:
\begin{center}
\texttt{K, M, d\_density, density\_clip, X, y}.
\end{center}

\begin{remark}
    As a final step, we recenter the resulting ANOVA components so that 
    $f_{\emptyset}$ captures the global mean. This operation comes at 
    negligible computational cost and merely translates the components, 
    leaving the structure of $\Xi$ and the hierarchical orthogonality 
    unchanged. Prior to this step, the components were already nearly 
    centered by construction; the residual offset is an empirical bias due to finite-sample approximation error, and correcting for it at essentially no cost is therefore worthwhile.
\end{remark}

\begin{remark}
    The $\tanh$ transformation, applied component-wise, maps the features into $[-1,1]^d$, as required by our framework. We emphasize that this transformation is applied \emph{solely within our explainer}: the predictive model is trained on the original data, and the baseline explainers we compare against likewise operate on the untransformed data. The comparison is therefore fair, since the underlying function of interest is identical across all methods; only its decomposition is computed in a different coordinate system.

    Moreover, by uniqueness of the functional ANOVA decomposition, this change of variables leaves the importance attributions invariant. Let $\mathbf{Z}$ denote the original features and $\mathbf{X} \coloneqq \tanh(\mathbf{Z})$ the corresponding transformation with $\tanh$ applied component-wise. Since $\tanh$ is a bijection, the marginal structure is preserved. Let $f$ be the function of interest defined on the original space, with the unique decomposition
    \begin{equation}
        f(\mathbf{Z}) = \sum_{S \subseteq [d]} f_S(\mathbf{Z}_S),
    \end{equation}
    and define
    \begin{equation}
        \widetilde{f}(\mathbf{X}) \coloneqq f\left(\tanh^{-1}(\mathbf{X})\right),
    \end{equation}
    which in turn admits a unique decomposition $\widetilde{f}(\mathbf{X}) = \sum_{S \subseteq [d]} \widetilde{f}_S(\mathbf{X}_S)$. Substituting $\mathbf{X} = \tanh(\mathbf{Z})$ yields
    \begin{equation}
        f(\mathbf{Z}) \;=\; \widetilde{f}(\tanh(\mathbf{Z})) \;=\; \sum_{S \subseteq [d]} \widetilde{f}_S\!\left(\tanh(\mathbf{Z}_S)\right).
    \end{equation}
    By uniqueness of this decomposition, we conclude that:
    \begin{equation}
        \forall S \subseteq [d], \quad \widetilde{f}_S( \mathbf X_S) = f_S(\mathbf{Z}_S).
    \end{equation}
\end{remark}

\section{Experiments details}\label{appendix:experiments}

\subsection{Resources}
All experiments were conducted on a MacBook Pro M4 with 32\,GB of RAM. 
The code is written in \texttt{Python} using mainly \texttt{numpy}, \texttt{scipy}, and \texttt{scikit-learn}. All computations reduce to standard linear-algebra operations, primarily array manipulation and the resolution of linear systems and are executed on CPU.

\subsection{Machine learning models}
We describe the four machine learning models used in the numerical 
experiments of this paper. The same architectures are used consistently 
across all experiments, including both the analytical setting and the 
real-world datasets. All of these machine learning models are treated as \emph{black boxes} by our method.

\paragraph{XGBoost.} We use \texttt{XGBoost} \cite{chen2015xgboost}, 
a gradient-boosted decision tree ensemble, with $100$ estimators, 
a maximum depth of $10$, and a learning rate of $0.05$. Both row and 
column subsampling are set to $0.8$ at each boosting iteration to mitigate 
overfitting. We use an $80/20$ train-test split, with early stopping after 
$30$ rounds without improvement on the held-out set. For classification 
tasks, the split is stratified and the objective is set to 
\texttt{binary:logistic} or \texttt{multi:softprob} depending on the 
number of classes; for regression, the default squared-error objective 
is used. The random seed is exposed as a pipeline parameter.

\paragraph{MLP.} We use a standard feed-forward multilayer perceptron with 
three hidden layers of widths $(128, 64, 32)$ and ReLU activations, 
followed by a dropout rate of $0.1$ applied after each hidden layer. 
Features are standardized using a \texttt{StandardScaler} 
fitted on the training set. The model is trained with the Adam 
optimizer~\cite{kingma2014adam}, a learning rate of $10^{-3}$, a weight 
decay of $10^{-4}$, a batch size of $256$, and a maximum of $500$ epochs, 
with early stopping after $30$ epochs without improvement on the held-out 
validation set. The loss function is set to mean squared error for 
regression, binary cross-entropy with logits for two-class classification, 
and categorical cross-entropy otherwise. We use an $80/20$ train-test 
split, stratified in the classification setting. The random seed is 
exposed as a pipeline parameter.

\paragraph{NAM.} Neural Additive Models \cite{agarwal2021neural} fit a 
generalized additive model of the form $h(\mathbf{x}) = C + \sum_{j=1}^{d} 
h_j(\mathbf x_j)$, where each \emph{shape function} $h_j$ is parameterized by an 
independent univariate subnetwork. Following the standard ReLU variant 
of \cite{agarwal2021neural}, each subnetwork consists of three hidden 
layers of widths $(64, 64, 32)$ with ReLU activations. Features are 
standardized on the training set prior to training. The model is 
trained with the Adam optimizer \cite{kingma2014adam}, a learning rate 
of $6.74 \times 10^{-3}$ decayed exponentially by a factor of $0.995$ 
per epoch, weight decay of $10^{-6}$, and an additional $L_2$ penalty 
of $10^{-3}$ on the shape function outputs to encourage smoothness. 
We use a batch size of $1024$, a maximum of $1000$ epochs, and early 
stopping after $50$ epochs without improvement on the held-out 
validation set. The loss function is mean squared error for regression, 
binary cross-entropy with logits for two-class classification, and 
categorical cross-entropy otherwise. After training, the shape functions 
are centered so that each $h_j$ has zero empirical mean over the training 
set, with the residual absorbed in the intercept $C$. We use an 
$80/20$ train-test split, stratified in the classification setting, 
and the random seed is exposed as a pipeline parameter.

\paragraph{EBM.} Explainable Boosting Machines \cite{nori2019interpretml} are second-order generalized additive models of the form $h(\mathbf{x}) = C + \sum_{j} h_j(\mathbf x_j) + \sum_{j,k} h_{j,k}(\mathbf x_j, \mathbf x_k)$, where each univariate and bivariate  function is fitted greedily using cyclic gradient boosting with shallow decision trees. We use the reference implementation provided by the \texttt{interpret} package \cite{nori2019interpretml}, with a maximum of $256$ bins for feature discretization, $10$ automatically selected pairwise interactions, and up to $5\,000$ boosting rounds with a learning rate of $10^{-2}$. Each base learner is a shallow tree with at most $3$ leaves and a minimum of $2$ samples per leaf. We use an $80/20$ train-test split, stratified in the classification setting. The random seed is exposed as a pipeline parameter.

\subsection{Explainers}

We fix in advance an evaluation set $\mathcal X_{\text{eval}}$ of shape $N_{\text{eval}} \times d$ on which all subsequent quantities are computed and compared. In the analytical setting, we explain the entire dataset, i.e., $N_{\text{eval}} = n = 10\,000$. For the real-world datasets, we sample $N_{\text{eval}} = 5\,000$ instances uniformly at 
random. For a given dataset, all methods are evaluated on the same set of instances to ensure a fair comparison.

\paragraph{KernelSHAP.} We fix in advance a background dataset of size 
$N_{\text{bg}}$ and apply the \texttt{KernelExplainer} from the 
\texttt{shap}\footnote{\url{https://github.com/shap/shap}} package \cite{lundberg2017unified} directly to the trained model. In the analytical setting, we apply this method directly to the function $f$, treated as a black box. For the real-world datasets, it is applied exclusively to the trained MLP.

\paragraph{DeepSHAP.} We fix in advance a background dataset of size 
$N_{\text{bg}}$ and apply the \texttt{DeepExplainer} from the 
\texttt{shap} package \cite{lundberg2017unified} directly to the trained model. For the real-world datasets, it is applied exclusively to the trained MLP.

\begin{remark}
    In practice, we set $N_{\text{bg}} = 200$ for the datasets \emph{Bike Sharing}, \emph{California Housing} and \emph{Pima Indians Diabetes}, a standard value in the literature. For the datasets \emph{Adult Census Income}, \emph{Electrical Grid Stability} and \emph{Superconductivity}, we use $N_{\text{bg}} = 50$ for computational reasons.
\end{remark}

\paragraph{TreeSHAP.} We apply TreeSHAP~\cite{lundberg2018consistent}, 
via the \texttt{TreeExplainer} class of the \texttt{shap} 
package, to the trained XGBoost model. We use the \texttt{tree\_path\_dependent} option, which computes Shapley values in the observational sense by accounting for feature dependencies through the tree structure.

\paragraph{TreeHFD.} We apply TreeHFD~\cite{benard2025tree}, via the 
\texttt{treehfd}\footnote{\url{https://github.com/ThalesGroup/treehfd}} package, 
to the same trained XGBoost model as TreeSHAP. We use the first-order 
mode (\texttt{interaction\_order=1}), which restricts the decomposition 
to main effects, both for computational efficiency and for consistency 
with our figures, in which we display the main effects.

\paragraph{EBM.} The additive decomposition is part of the fitted EBM itself. For each evaluation instance $\mathbf{x}$ and each feature $j$, the main effect $h_j(\mathbf x_j)$ is extracted directly via the \texttt{eval\_terms} method of the \texttt{interpretML}\footnote{\url{https://interpret.ml/}} package \cite{nori2019interpretml}, which returns the contribution of every univariate (and bivariate) shape function to the model output. Displayed main effects are obtained by plotting $h_j(x_j)$ against $x_j$ over the evaluation set. In the classification setting, all contributions are reported on the logit scale.

\paragraph{NAM.} Likewise, the additive decomposition is part of the NAM 
architecture. For each evaluation instance $\mathbf{x}$ and each feature 
$j$, the main effect $h_j(\mathbf x_j)$ is obtained by a forward pass of the 
$j$-th subnetwork on the scalar input $x_j$. After training, the shape 
functions are centered so that $\mathbb{E}[h_j( \mathbf X_j )] = 0$ on the training 
set, with the residual absorbed in the intercept 
$C$~\cite{agarwal2021neural}. Displayed main effects are obtained as for 
EBM, by plotting $h_j(\mathbf x_j)$ against $\mathbf x_j$ over the evaluation set; in 
classification, they are reported on the logit scale.

\subsection{Analytical Case}\label{appendix:analytical}

\paragraph{Theoretical decomposition.} In this case, we consider a random vector $\mathbf{X} \in [-1,1]^3$ whose joint density belongs to the Farlie--Gumbel--Morgenstern \citep{nelsen2006introduction} family and is given for $\mathbf x \in [-1,1]^3$ by:
\begin{equation}
    p_{\rho}(\mathbf{x}) = \frac{1}{8}\bigl(1 + \rho\,({x}_1 {x}_2 + {x}_2 {x}_3 + {x}_1 {x}_3)\bigr),
\end{equation}
where $\rho \in (-\tfrac{1}{3},\,1)$ is the dependence parameter. The marginal densities of order 1 and 2 are explicit and are given by
\begin{align}
    p_j(x_j) &= \frac{1}{2}, \\
    p_{i,j}({x}_i, {x}_j) &= \frac{1}{4}\bigl(1 + \rho\, {x}_i\, {x}_j\bigr).
\end{align}
Each marginal is uniform on $[-1,1]$, and the pairwise correlation satisfies $\operatorname{Corr}({X}_i, {X}_j) = \rho/3$ for all $i \neq j$. Note that for any $\rho \in (-1/3 , 1)$, we have:
\begin{equation}
    0 < \underbrace{\frac{1}{8}(1 - \rho)}_{c_1} \leq p_{\rho} \leq \underbrace{\frac{1}{8}( 1 + 3 \rho )}_{c_2} < \infty,
\end{equation}
which satisfies the Assumption \ref{assu:densite}. We define the theoretical components of $f$ in terms of elements of $\Xi$ as follows:
\begin{align}
    f_1(  x_1 ) &\coloneqq \frac{\sqrt{\tfrac{2}{7}}\, \phi_3({x}_1) - \sqrt{\tfrac{2}{3}}\, \phi_1({x}_1) }{p_1(  x_1 )}, \\
    f_2(  x_2 ) &\coloneqq \frac{ \sqrt{\tfrac{2}{3}}\, \phi_1({x}_2) + \sqrt{\tfrac{2}{5}}\, \phi_2({x}_2) }{p_2(  x_2 )}, \\
    f_{1,2}(  x_1 ,  x_2 ) &\coloneqq \frac{\tfrac{2}{9}\, \phi_4({x}_1) \phi_4({x}_2) + \tfrac{2}{17}\, \phi_8({x}_1) \phi_8({x}_2) }{p_{1,2}(  x_1 ,  x_2 )},
\end{align}
and $f_{\emptyset} = f_{3} = f_{1,3} = f_{2,3} = f_{1,2,3} = 0$.

\paragraph{Simulations.}
In our theoretical setup, the function $f$ is available in closed form, and samples of $\mathbf{X}$ can be drawn via rejection sampling from the uniform distribution on $[-1,1]^3$ using the density $p_{\rho}$. We set $\rho = 1/2$ for simplicity. The illustration in 
Fig~\ref{fig:AC} is produced with $n = 10\,000$ samples.

\paragraph{Comparisons.}
We compare the main effects $f_1, f_2, f_3$ against several 
explanation methods (KernelSHAP, TreeSHAP, TreeHFD) and interpretable models (NAM and EBM). In each case, we follow the procedure described above and explain 
the $10\,000$ instances of the dataset.
\begin{itemize}
    \item \textbf{KernelSHAP.} We apply KernelSHAP directly to $f$, 
    treated as a \emph{black-box} model.
    \item \textbf{TreeSHAP.} We train an XGBoost model and apply TreeSHAP 
    to the trained estimator.
    \item \textbf{TreeHFD.} We apply TreeHFD to the same XGBoost model.
    \item \textbf{NAM.} We train a NAM and report its learned main effects.
    \item \textbf{EBM.} We train an EBM and report its learned main effects.
\end{itemize}

\begin{remark}
    Note that we apply our own estimator on this synthetic dataset, assuming that the theoretical $f$ is black-box. And we verify that our method also captures the true underlying ANOVA structure (see Fig.\ref{fig:AC_vs_ours}).
\end{remark}

\begin{figure}[ht]
  \centering
  \includegraphics[width=\textwidth]{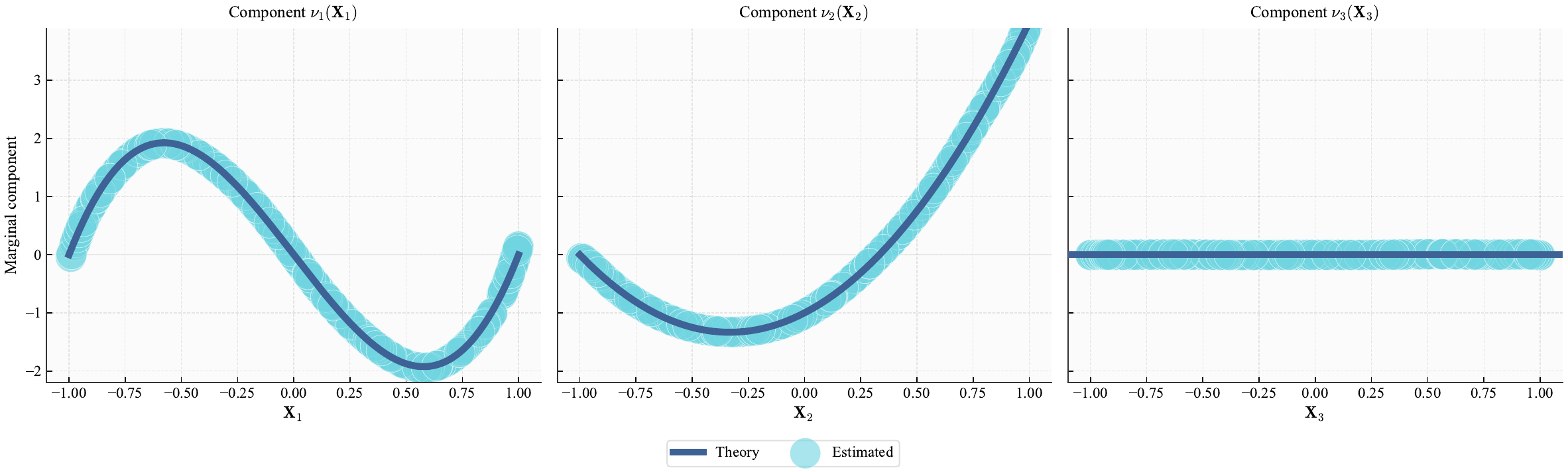}
  \caption{Estimated main effects in the analytical setting for 
${X}_1$, ${X}_2$, and the irrelevant variable ${X}_3$. We take $K=2$, $M=10$, $d_{\text{density}} = 10$ and $\varepsilon = 0.01$}
  \label{fig:AC_vs_ours}
\end{figure}

\subsection{Datasets}

The \emph{California Housing} dataset is directly available from \texttt{scikit-learn}. The remaining datasets are retrieved from the UCI Machine Learning 
Repository~\cite{dua2017uci} using the \texttt{fetch\_ucirepo} function of the \texttt{ucimlrepo}\footnote{\url{https://pypi.org/project/ucimlrepo/}} package. All features are used for every dataset, with the exception of \emph{Bike Sharing}, where we retain only eight: \texttt{hr}, \texttt{weekday}, \texttt{holiday}, \texttt{season}, \texttt{atemp}, \texttt{hum}, \texttt{windspeed}, and \texttt{weathersit} following \cite{benard2025tree}.

\subsection{Hyperparameters}

In this section, we specify the hyperparameters used by our method on each dataset. Recall that these hyperparameters are:
\begin{enumerate}
    \item $K$: the truncation order of the interaction levels;
    \item $M$: the maximum polynomial degree (numerator);
    \item $d_{\text{density}}$: the maximum polynomial degree used to estimate the marginal densities (denominator);
    \item $\varepsilon$: the clipping threshold applied to the 
    estimated densities.
\end{enumerate}

\begin{table}[ht]
  \caption{Hyperparameters of our method for all models across all 
  datasets. The -- entries correspond to datasets on which the 
  NAM baseline was not trained for computational reasons.}
  \label{table:hyperparam}
  \centering
  \resizebox{\textwidth}{!}{%
  \begin{tabular}{l cccc cccc cccc cccc}
    \toprule
    & \multicolumn{4}{c}{XGBoost} 
    & \multicolumn{4}{c}{MLP} 
    & \multicolumn{4}{c}{EBM} 
    & \multicolumn{4}{c}{NAM} \\
    \cmidrule(lr){2-5} \cmidrule(lr){6-9} \cmidrule(lr){10-13} \cmidrule(lr){14-17}
    \textbf{ID} 
    & $K$ & $M$ & $d_{\text{density}}$ & $\varepsilon$ 
    & $K$ & $M$ & $d_{\text{density}}$ & $\varepsilon$ 
    & $K$ & $M$ & $d_{\text{density}}$ & $\varepsilon$ 
    & $K$ & $M$ & $d_{\text{density}}$ & $\varepsilon$ \\
    \midrule
    \textbf{BS} & 2 & 10 & 5 & 0.01 
                & 2 & 8  & 6 & 0.01 
                & 2 & 10 & 4 & 0.01 
                & -- & -- & -- & -- \\
    \textbf{CH} & 2 & 10 & 4 & 0.01 
                & 1 & 10 & 4 & 0.01 
                & 2 & 10 & 4 & 0.01 
                & 1 & 10 & 4 & 0.01 \\
    \textbf{CI} & 1 & 10 & 4  & 0.01 
                & 1 & 10 & 4  & 0.01 
                & 1 & 10 & 6  & 0.01
                & -- & -- & -- & -- \\
    \textbf{EG} & 2 & 4  & 4  & 0.01 
                & 2 & 4  & 4  & 0.01 
                & 2 & 4  & 4  & 0.01 
                & 1 & 4  & 4  & 0.01 \\
    \textbf{PI} & 2 & 5  & 4  & 0.1
                & 2 & 5  & 4  & 0.01 
                & 2 & 5  & 4  & 0.01
                & 1 & 10 & 4  & 0.01 \\
    \textbf{SC} & 1 & 10 & 10 & 0.01 
                & 1 & 10 & 10 & 0.01  
                & 1 & 10 & 10 & 0.01  
                & -- & -- & -- & -- \\
    \bottomrule
  \end{tabular}
  }
\end{table}
\begin{remark}
In practice, we always take $K \in \{1, 2\}$, so that our method reduces 
to a GAM ($K=1$) or a GA\textsuperscript{$2$}M ($K=2$). The density-estimation degree $d_{\text{density}}$ can also be kept small: a low-degree polynomial is 
typically sufficient to capture the overall \emph{trend} of the marginal 
densities. Finally, we never go beyond $M = 10$ for the shape-function 
polynomials, as we empirically observe no additional signal capture 
beyond this degree. Pushing $M$ further can even degrade the 
hierarchical orthogonality of the decomposition, a signal that either 
a more sophisticated penalized model selection scheme or an increase 
of the truncation order to $K = 3$ would be required.
\end{remark}

\subsection{Performance metrics}\label{appendix:performance_metrics}
In this subsection, we report all performance metrics that were 
omitted from the main body of the paper.

\paragraph{Tables.} For clarity, we reproduce here the tables reporting the dataset characteristics and the performance of our method across the four models and all datasets. Tab~\ref{table:perf_xgb_mlp} and Tab~\ref{table:perf_ebm_nam} report the predictive performance of the trained models on the six real-world datasets, together with the corresponding metrics of our method.

\begin{table}[ht]
  \caption{Dataset characteristics and performance of our method.}
  \label{table:perf_xgb_mlp}
  \centering
  \begin{tabular}{lrrcrcrcrcrr}
    \toprule
    & & & \multicolumn{4}{c}{XGB} & \multicolumn{4}{c}{MLP} \\
    \cmidrule(lr){4-7} \cmidrule(lr){8-11}
    \textbf{ID}  & $n$ & $d$ & Acc. & $K$ & Rec. & Ortho. & Acc. & $K$ & Rec. & Ortho. \\
    \midrule
    \textbf{BS} & 17\,379 & 8       & 0.86 & 2 & 0.92 & $9.49\mathrm{e}{-2}$  & 0.85 & 2 & 0.91 & $8.83\mathrm{e}{-2}$ \\
    \textbf{CH} & 20\,640 & 8     & 0.84 & 2 & 0.86 & $6.35\mathrm{e}{-2}$  & 0.81 & 1 & 0.88 & 0.00 \\
    \textbf{CI} & 48\,842 & 14       & 0.88 & 1 & 0.91 & 0.00  & 0.85 & 1 & 0.94 & 0.00 \\
    \textbf{EG} & 10\,000 & 12       & 0.93 & 2 & 0.89 & $7.41\mathrm{e}{-3}$  & 0.98 & 2 & 0.95 & $6.61\mathrm{e}{-3}$ \\
    \textbf{PI} & 768 & 8       & 0.74 & 2 & 0.85 & $9.56\mathrm{e}{-2}$  & 0.73 & 2 & 0.99 & 0.00 \\
    \textbf{SC} & 21\,263 & 81       & 0.93 & 1 & 0.86 & 0.00  & 0.91 & 1 & 0.89 & 0.00 \\
    \bottomrule
  \end{tabular}
\end{table}

\begin{table}[ht]
  \caption{Dataset characteristics and performance of our method on EBM and NAM.}
  \label{table:perf_ebm_nam}
  \centering
  \begin{tabular}{lrrcrcrcrcrr}
    \toprule
    & & & \multicolumn{4}{c}{EBM} 
        & \multicolumn{4}{c}{NAM} \\
    \cmidrule(lr){4-7} \cmidrule(lr){8-11}
    \textbf{ID} & $n$ & $d$ 
                & Acc. & $K$ & Rec. & Ortho. 
                & Acc. & $K$ & Rec. & Ortho. \\
    \midrule
    \textbf{BS} & 17\,379 & 8  & 0.83 & 2 & 0.96 & $9.52\mathrm{e}{-2}$ 
                               & --   & -- & --  & --                  \\
    \textbf{CH} & 20\,640 & 8  & 0.83 & 2 & 0.89 & $4.04\mathrm{e}{-2}$ 
                               & 0.76 & 1 & 0.94 & 0.00                \\
    \textbf{CI} & 48\,842 & 14 & 0.88 & 1 & 0.90 & 0.00 
                               & --   & -- & --  & --                  \\
    \textbf{EG} & 10\,000 & 12 & 0.94 & 2 & 0.99 & $5.58\mathrm{e}{-3}$ 
                               & 0.88 & 1 & 1.00 & 0.00 \\
    \textbf{PI} & 768     & 8  & 0.76 & 2 & 0.97 & 0.00 
                               & 0.73 & 1 & 1.00 & 0.00                \\
    \textbf{SC} & 21\,263 & 81 & 0.90 & 1 & 0.88 & 0.00 
                               & --   & -- & --  & --                  \\
    \bottomrule
  \end{tabular}
\end{table}

\emph{Acc.} denotes the predictive performance of the trained model: classification accuracy for classification tasks and the coefficient of determination for regression tasks. The quantity \emph{Rec.} is the reconstruction coefficient of determination of our estimator, i.e. the fraction of the output variance captured by our decomposition on the entire dataset. Finally, \emph{Ortho.} is defined as the following absolute cosine
\begin{equation}
    \text{Ortho.} \coloneqq \max\limits_{S , T} \ \left| \mathrm{cos}\left( \widehat{f}_S , \widehat{f}_T \right) \right|.
\end{equation}
The cosine is formally defined as follows:
\begin{equation}
    \mathrm{cos}\left( U , V \right) \coloneqq \frac{ \left\langle U , V \right\rangle }{ \| U \| \| V\| } =  \frac{ \mathbb E\left[ U  V \right] }{ \sqrt{ \mathbb E\left[ U^2 \right] } \sqrt{ \mathbb E\left[ V^2 \right] } }.
\end{equation}
The absolute cosine measures the deviation from orthogonality : the more the quantity is approaching zero, the more the \emph{angle} is near to the orthogonality. Finally, for significance reasons, the max is taken over the following ensemble:
\begin{equation}
    \{ (S, T) : T \subsetneq S,\ \mathbb{V}(\widehat{f}_S(\mathbf{X}_S)) / \mathbb{V}(f(\mathbf{X})) \geq 1\% \}.
\end{equation}
We restrict the maximum to interaction components whose variance represents at least $1\%$ of the output variance (following \cite{benard2025tree}), so as to discard components with negligible contribution to the decomposition.

\paragraph{Computation time.} The computation times of our method on each real-world dataset are reported in Tab.~\ref{table:time}. For each dataset, we fix the random seed to 42 and train the corresponding machine learning model. We then apply our decomposition method to each trained model over 10 independent runs, and report the mean and standard deviation. We observe that our method is extremely fast, enabling the explanation of hundreds or thousands of instances in a few seconds or less. For instance, the quasi-additive structure of an MLP trained on the \emph{Adult Census Income} dataset allows us to explain all 48\,842 instances in approximately 0.6\,s. We display in the Table \ref{table:time} the computation time to obtain the entire truncated decomposition for each dataset.

\begin{table}[ht]
\newcommand{\val}[2]{$#1{\scriptstyle\,\pm\,#2}$}
  \caption{Time (in s) to estimate the decomposition on the entire dataset (\val{\text{Mean}}{\text{std}} over 10 repetitions).}
  \label{table:time}
  \centering
  \begin{tabular}{lrrcccc}
    \toprule
    \textbf{ID} & $n$ & $d$ & XGB & MLP & EBM & NAM \\
    \midrule
    \textbf{BS} & 17\,379 & 8  & \val{9.806}{0.228}  & \val{5.741}{0.182} & \val{9.721}{0.206}  & --                 \\
    \textbf{CH} & 20\,640 & 8  & \val{11.048}{0.415} & \val{0.129}{0.012} & \val{11.077}{0.414} & \val{0.126}{0.005} \\
    \textbf{CI} & 48\,842 & 14 & \val{0.575}{0.023}  & \val{0.597}{0.053} & \val{0.691}{0.010}  & --                 \\
    \textbf{EG} & 10\,000 & 12 & \val{1.420}{0.048}  & \val{1.407}{0.025} & \val{1.365}{0.019}  & \val{0.039}{0.007} \\
    \textbf{PI} & 768     & 8  & \val{0.237}{0.006}  & \val{0.230}{0.009} & \val{0.230}{0.006}  & \val{0.013}{0.001} \\
    \textbf{SC} & 21\,263 & 81 & \val{3.520}{0.097}  & \val{3.547}{0.159} & \val{3.501}{0.053}  & --                 \\
    \bottomrule
  \end{tabular}
\end{table}

\paragraph{Hierarchical orthogonality.} As a consequence of the 
hierarchical orthogonality constraint~\eqref{eq:fanova_orthogonality}, 
the components $\{f_S(\mathbf{X}_S)\}_{S \subseteq [d]}$ must satisfy
\begin{equation}
    \forall S, T \subseteq [d], \quad T \subsetneq S \implies 
    \mathbb{E}\!\left[ f_S(\mathbf{X}_S)\, f_T(\mathbf{X}_T) \right] = 0.
\end{equation}
In practice, we measure the empirical cosine between pairs of components 
$f_S$ and $f_T$ with $T \subsetneq S \subseteq [d]$. When 
$T = \emptyset$, the cosine is defined by:
\begin{equation}
    \frac{ \mathbb{E}\!\left[ f_S(\mathbf{X}_S) \right] }
         { \sqrt{ \mathbb{E}\!\left[ f_S(\mathbf{X}_S)^2 \right] } },
\end{equation}
which quantifies the degree of centering of $f_S$. For 
$T \neq \emptyset$, we compute
\begin{equation}
    \frac{ \mathbb{E}\!\left[ f_S(\mathbf{X}_S)\, 
    f_T(\mathbf{X}_T) \right] }
    { \sqrt{ \mathbb{E}\!\left[ f_S(\mathbf{X}_S)^2 \right] 
    \, \mathbb{E}\!\left[ f_T(\mathbf{X}_T)^2 \right] } }.
\end{equation}
After estimation, we systematically recenter the estimated components 
$\widehat{f}_S$, so that for every $S \subseteq [d]$,
\begin{equation}
    \frac{ \mathbb{E}\!\left[ \widehat{f}_S(\mathbf{X}_S) \right] }
         { \sqrt{ \mathbb{E}\!\left[ \widehat{f}_S(\mathbf{X}_S)^2 \right] } } = 0.
\end{equation}
As a result, for every $\emptyset \subsetneq T \subsetneq S \subseteq [d]$, 
the empirical cosine between $\widehat{f}_S$ and $\widehat{f}_T$ 
coincides with their empirical correlation, since both components are 
centered. Naturally, \emph{Ortho.} is computed from this quantity.
\begin{remark}
    Although the theoretical decomposition provably satisfies the 
    constraint~\eqref{eq:fanova_orthogonality}, exact hierarchical orthogonality is difficult to achieve in practice, especially on finite samples 
    combined with imperfect density estimation.
\end{remark}
For the analytical setting, we report in Tab~\ref{table:ortho_analytical} 
the empirical cosines, in absolute value, between the theoretical ANOVA 
components. The hierarchical orthogonality constraint is well satisfied 
in this case: all corresponding cosines are below $10^{-2}$, confirming 
numerically the theoretical result.
\begin{table}[ht]
  \newcommand{\val}[2]{$#1{\scriptstyle\,\pm\,#2}$}
  \caption{Empirical cosines between theoretical ANOVA components on 
  the analytical setting, reported as mean $\pm$ standard deviation 
  over $100$ independent runs with $n = 100\,000$ samples each.}
  \label{table:ortho_analytical}
  \centering
  \begin{tabular}{cc}
    \toprule
    \textbf{$(S, T)$} & Absolute cosine \\
    \midrule
    $(\{1\},    \emptyset)$ & \val{0.0027}{0.0018} \\
    $(\{2\},    \emptyset)$ & \val{0.0026}{0.0019} \\
    $(\{1, 2\}, \emptyset)$ & \val{0.0024}{0.0020} \\
    $(\{1, 2\}, \{1\})$     & \val{0.0022}{0.0017} \\
    $(\{1, 2\}, \{2\})$     & \val{0.0028}{0.0024} \\
    \bottomrule
  \end{tabular}
\end{table}

\paragraph{Explained variance.} Across all experiments, the trained 
models prove to be highly additive, which explains the consistently 
high $R^2$ reconstruction values reported in Tab~\ref{sample-table} and Tab~\ref{table:perf_ebm_nam}. However, our decomposition does not achieve a perfectly unitary $R^2$, for three reasons. \emph{First}, the truncated decomposition is restricted to interaction orders $K \in \{1, 2\}$ and therefore cannot capture signal carried by higher-order interactions, whenever the underlying model exhibits them. \emph{Second}, our estimator can still be refined: a more accurate density estimator or a more sophisticated (\emph{e.g.}, penalized) strategy for the regularized least-squares problem would improve the fraction of variance recovered. \emph{Third}, and more fundamentally, recovering exactly $100\%$ of the variance is not a desirable target for an explainability method. Methods that enforce the \emph{efficiency} axiom by construction, such as the SHAP family, distribute the totality of the prediction over feature attributions. This is visible in our experiments: SHAP-based estimates produce substantially noisy curves. By relaxing the efficiency requirement, our approach forgoes explaining the last few percent of variance in exchange for attributions that track the genuine additive structure of the signal rather than its residual noise.

\subsection{Additional Analytical Case}\label{appendix:additional_ac}

We provide a supplementary analytical example in an unbounded-density setting, where Assumption~\ref{assu:densite} no longer holds. We observe that the results of this experiment are keeping relevant. Let consider $\mathbf{X} \in [-1,1]^3$ with ${X}_j \coloneqq \tanh({Z}_j)$ for $j \in [3]$, where $\mathbf{Z} \sim \mathcal{N}(\mathbf{0}, \Sigma)$ and $\Sigma$ has unit diagonal and constant off-diagonal $\rho > 0$. The marginal densities are
\begin{align}
    p_j(x) &= \frac{1}{\sqrt{2\pi}\,(1-x^2)} \exp\!\left(-\tfrac{1}{2}{\sigma}^2(x)\right), \\
    p_{i,j}(x_i, x_j) &= \frac{1}{2\pi\sqrt{1-\rho^2}\,(1-x_i^2)(1-x_j^2)} \nonumber\\
    &\quad\times \exp\!\left(-\frac{{\sigma}^2(x_i) - 2\rho\, \sigma(x_i)\, \sigma(x_j) + {\sigma}^2(x_j)}{2(1-\rho^2)}\right),
\end{align}
where $\sigma \coloneqq \operatorname{arctanh}$. The theoretical ANOVA components are defined as in the first analytical case; only the denominators change due to the different densities. We set $\rho = 1/2$ and draw samples by generating $\mathbf{Z}^{(1)}, \dots, \mathbf{Z}^{(n)} \overset{\text{i.i.d.}}{\sim} \mathcal{N}(\mathbf{0}, \Sigma)$ and applying $\tanh$ coordinate-wise.

Fig~\ref{fig:AC_gauss} shows the estimated main effects with $n = 10\,000$ samples: all methods closely track the theoretical components on ${X}_1$ and ${X}_2$, and correctly assign near-zero contribution to ${X}_3$. Fig~\ref{fig:AC_vs_ours_gauss} displays the comparison with our estimator, treating $f$ as a black box. Finally, Tab~\ref{table:ortho_analytical_bis} reports the empirical cosines between ANOVA components, confirming that hierarchical orthogonality holds empirically despite the violation of Assumption~\ref{assu:densite}.

\begin{figure}[ht]
  \centering
  \includegraphics[width=\textwidth]{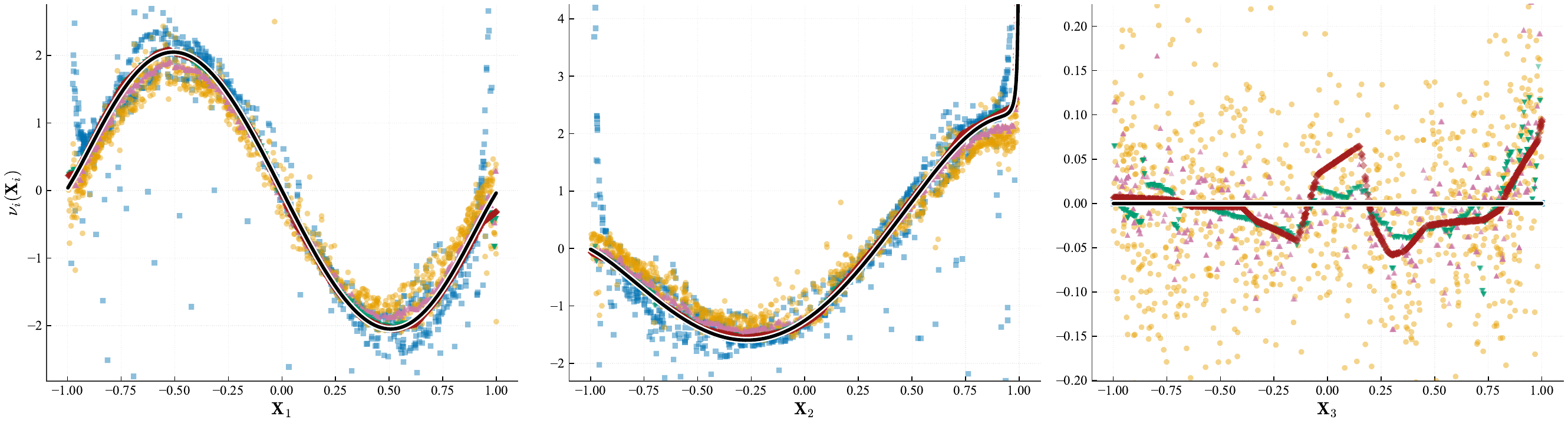}
  \caption{Estimated main effects in the unbounded-density setting for 
${X}_1$, ${X}_2$ and ${X}_3$. Comparison with 
\textcolor[HTML]{0072B2}{KernelSHAP}, 
\textcolor[HTML]{E69F00}{TreeSHAP}, 
\textcolor[HTML]{CC79A7}{TreeHFD}, 
\textcolor[HTML]{A31A1A}{NAM} and 
\textcolor[HTML]{009E73}{EBM}.}
  \label{fig:AC_gauss}
\end{figure}

\begin{figure}[ht]
  \centering
  \includegraphics[width=\textwidth]{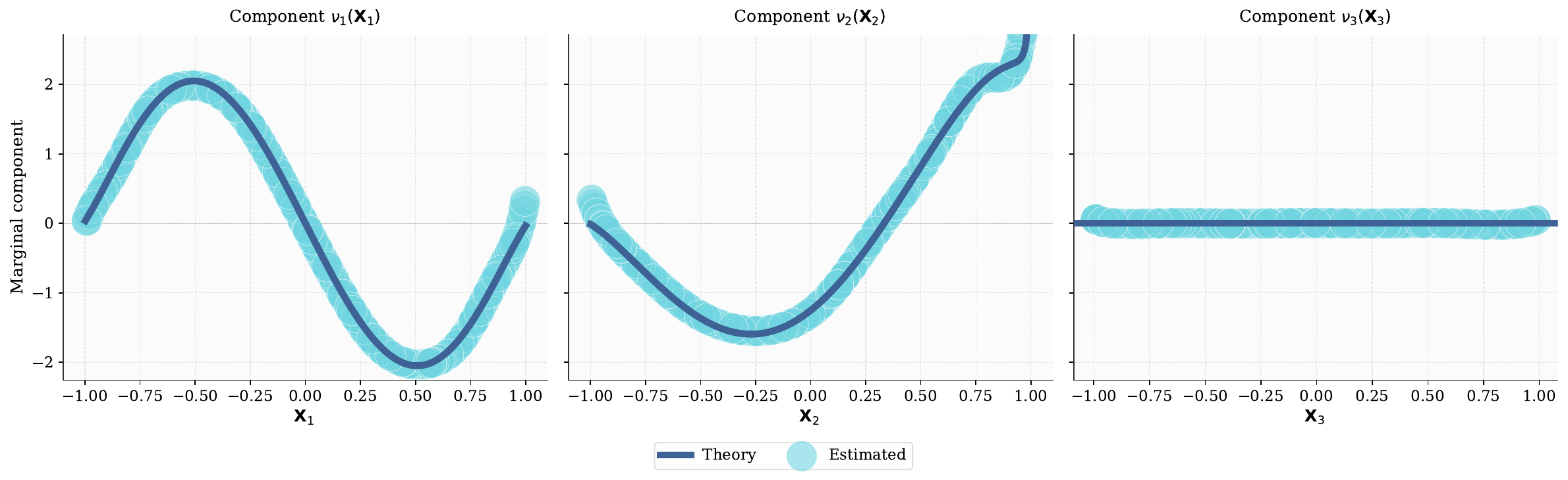}
  \caption{Comparison between the theoretical ANOVA components and our estimator ($K=2$, $M=10$, $d_{\text{density}} = 10$, $\varepsilon = 0.01$) in the unbounded-density setting.}
  \label{fig:AC_vs_ours_gauss}
\end{figure}

\begin{table}[H]
  \newcommand{\val}[2]{$#1{\scriptstyle\,\pm\,#2}$}
  \caption{Empirical cosines between theoretical ANOVA components in the unbounded-density setting (mean $\pm$ std over $100$ runs, $n = 100\,000$).}
  \label{table:ortho_analytical_bis}
  \centering
  \begin{tabular}{cc}
    \toprule
    \textbf{$(S, T)$} & Absolute cosine \\
    \midrule
    $(\{1\},    \emptyset)$ & \val{0.0023}{0.0017} \\
    $(\{2\},    \emptyset)$ & \val{0.0025}{0.0019} \\
    $(\{1, 2\}, \emptyset)$ & \val{0.0033}{0.0020} \\
    $(\{1, 2\}, \{1\})$     & \val{0.0017}{0.0013} \\
    $(\{1, 2\}, \{2\})$     & \val{0.0087}{0.0132} \\
    \bottomrule
  \end{tabular}
\end{table}

\section{Connection with SHAP-based estimators}
\label{app:global-importance}

This section complements the qualitative comparisons of local explanations with global feature-importance results on five datasets: Bike Sharing (\textbf{BS}), California Housing (\textbf{CH}), Census Income (\textbf{CI}), Electrical Grid (\textbf{EG}), and Pima Indian Diabetes (\textbf{PI}). We compare our method with TreeSHAP for XGBoost models, and with KernelSHAP and DeepSHAP for multilayer perceptrons (MLPs). Each comparison is performed within the same model class. We omit Superconductivity (\textbf{SC}) from this additional comparison for computational reasons and to keep the presentation readable.

\paragraph{Evaluation sample.}
For computational reasons, we compute global importance scores on an evaluation sample of $N_{\text{eval}} = 5\,000$ points for \textbf{BS}, \textbf{CH}, \textbf{CI}, and \textbf{EG}. For \textbf{PI}, we use all $N_{\text{eval}} = 768$ available observations. For each dataset, all methods use the same evaluation points $(\mathbf{x}^{(i)})_{i=1}^{N_{\text{eval}}}$.

\paragraph{Global importance.}
For a given method and trained model, let $A \in \mathbb{R}^{N_{\text{eval}} \times d}$ be the matrix of feature contributions on the evaluation sample, with
\begin{equation}
    A_{ij} =
    \begin{cases}
        \widehat{f}_{\{j\}}(x_j^{(i)}),
            & \text{the estimated main effect}, \\
        \phi_j(\mathbf{x}^{(i)}),
            & \text{the SHAP-based explainer},
    \end{cases}
    \label{eq:global-contribution-matrix}
\end{equation}
where $\phi_j$ denotes the SHAP attribution associated with feature $j$ for the corresponding explainer and model. Thus, our scores summarize the estimated main effects. For each feature $j$, we compute its empirical $L_1$ magnitude and normalize it across features as follows
\begin{equation}
    m_j(A) \coloneqq \frac{1}{N_{\text{eval}}} \sum_{r=1}^{N_{\text{eval}}}|A_{rj}|,
    \qquad
    I_j(A) \coloneqq 100 \times \frac{m_j(A)}{\sum_{k=1}^{d}m_k(A)},
    \label{eq:global-l1-importance}
\end{equation}
provided that the denominator is positive. The resulting scores are
nonnegative and sum to $100$. Tables~\ref{tab:global-bs}--\ref{tab:global-pi} report these normalized mean absolute contributions. The notation $X_j$ identifies the $j$-th input feature in each dataset.

\paragraph{Rank agreement.}
For each method, we rank features by their global importance scores $I_j(A)$. Spearman correlation is computed across the $d$ features as the Pearson correlation between the two rank vectors, assigning average ranks to ties. The normalization in Eq.~\eqref{eq:global-l1-importance} does not change the ranking of the mean absolute contributions.

Table~\ref{tab:global-spearman} reports these correlations between our method and the corresponding SHAP-based explainer. TreeSHAP is compared with our XGB results, whereas KernelSHAP and DeepSHAP are compared with our MLP results. The final row is the unweighted mean across the five datasets. The mean correlations are $0.9130$, $0.9343$, and $0.9300$, respectively, indicating strong overall agreement in feature rankings. The lowest correlation is $0.7619$, obtained on \textbf{BS} for the comparison with TreeSHAP. These results quantify agreement in global rankings; they do not establish equivalence of local attributions or accuracy against a ground-truth explanation.

\begin{table}[!htbp]
    \centering
    \caption{Global feature importance (\%) on Bike Sharing (\textbf{BS}).}
    \label{tab:global-bs}
    \small
    \setlength{\tabcolsep}{7pt}
    \begin{tabular}{@{}lrrrrr@{}}
        \toprule
        & \multicolumn{2}{c}{XGBoost}
        & \multicolumn{3}{c}{MLP} \\
        \cmidrule(lr){2-3}
        \cmidrule(lr){4-6}
        Feature & Ours & TreeSHAP & Ours & KernelSHAP & DeepSHAP \\
        \midrule
        $X_{1}$ & 50.3744 & 45.5180 & 50.5807 & 47.4075 & 53.9525 \\
        $X_{2}$ & 1.5195 & 9.9017 & 2.1110 & 10.2345 & 8.6640 \\
        $X_{3}$ & 1.0339 & 0.6807 & 0.2932 & 0.7667 & 0.7717 \\
        $X_{4}$ & 10.7945 & 9.2791 & 7.9111 & 7.5666 & 7.2755 \\
        $X_{5}$ & 20.3323 & 20.2682 & 22.7251 & 20.5042 & 17.0642 \\
        $X_{6}$ & 10.3583 & 8.8610 & 12.0030 & 7.8736 & 7.2689 \\
        $X_{7}$ & 1.8881 & 2.3650 & 1.2493 & 1.8932 & 1.6535 \\
        $X_{8}$ & 3.6991 & 3.1264 & 3.1264 & 3.7538 & 3.3497 \\
        \bottomrule
    \end{tabular}
\end{table}

\begin{table}[!htbp]
    \centering
    \caption{Global feature importance (\%) on California Housing
    (\textbf{CH}).}
    \label{tab:global-ch}
    \small
    \setlength{\tabcolsep}{7pt}
    \begin{tabular}{@{}lrrrrr@{}}
        \toprule
        & \multicolumn{2}{c}{XGBoost}
        & \multicolumn{3}{c}{MLP} \\
        \cmidrule(lr){2-3}
        \cmidrule(lr){4-6}
        Feature & Ours & TreeSHAP & Ours & KernelSHAP & DeepSHAP \\
        \midrule
        $X_{1}$ & 33.2500 & 29.2700 & 21.8700 & 19.4600 & 11.3300 \\
        $X_{2}$ & 4.8300 & 4.2100 & 3.5700 & 2.9700 & 1.6400 \\
        $X_{3}$ & 5.4900 & 7.3600 & 2.9400 & 9.8500 & 3.6800 \\
        $X_{4}$ & 4.7700 & 1.7100 & 2.4900 & 2.0600 & 1.3800 \\
        $X_{5}$ & 1.4400 & 1.4700 & 1.1200 & 1.4900 & 0.7300 \\
        $X_{6}$ & 14.6200 & 13.5500 & 10.6000 & 12.4700 & 6.5800 \\
        $X_{7}$ & 17.5400 & 23.0300 & 30.9100 & 28.6800 & 39.1100 \\
        $X_{8}$ & 18.0600 & 19.3800 & 26.5000 & 23.0100 & 35.5600 \\
        \bottomrule
    \end{tabular}
\end{table}

\begin{table}[!htbp]
    \centering
    \caption{Global feature importance (\%) on Census Income (\textbf{CI}).}
    \label{tab:global-ci}
    \small
    \setlength{\tabcolsep}{7pt}
    \begin{tabular}{@{}lrrrrr@{}}
        \toprule
        & \multicolumn{2}{c}{XGBoost}
        & \multicolumn{3}{c}{MLP} \\
        \cmidrule(lr){2-3}
        \cmidrule(lr){4-6}
        Feature & Ours & TreeSHAP & Ours & KernelSHAP & DeepSHAP \\
        \midrule
        $X_{1}$ & 13.3054 & 15.2164 & 14.9586 & 17.0881 & 17.0826 \\
        $X_{2}$ & 6.8988 & 1.4843 & 1.3829 & 1.3580 & 1.9452 \\
        $X_{3}$ & 1.1135 & 1.6579 & 0.8368 & 1.0403 & 1.5745 \\
        $X_{4}$ & 4.8997 & 0.9003 & 8.9546 & 4.7806 & 4.6948 \\
        $X_{5}$ & 12.0810 & 13.5861 & 13.6308 & 14.8325 & 12.4089 \\
        $X_{6}$ & 18.5472 & 18.5625 & 19.8913 & 16.5109 & 16.6904 \\
        $X_{7}$ & 10.1375 & 7.0298 & 2.5592 & 2.8031 & 3.1396 \\
        $X_{8}$ & 8.4870 & 14.1087 & 4.3472 & 8.7373 & 9.1030 \\
        $X_{9}$ & 0.7372 & 0.5729 & 1.4692 & 1.6415 & 1.8181 \\
        $X_{10}$ & 4.8173 & 2.6871 & 7.5560 & 7.2096 & 6.1510 \\
        $X_{11}$ & 8.9345 & 11.8062 & 12.6143 & 10.0592 & 10.5772 \\
        $X_{12}$ & 2.5121 & 3.8328 & 1.9152 & 1.7506 & 2.6594 \\
        $X_{13}$ & 7.1106 & 8.0276 & 9.4338 & 11.7770 & 11.5470 \\
        $X_{14}$ & 0.4182 & 0.5275 & 0.4502 & 0.4113 & 0.6083 \\
        \bottomrule
    \end{tabular}
\end{table}

\begin{table}[!htbp]
    \centering
    \caption{Global feature importance (\%) on Electrical Grid
    (\textbf{EG}).}
    \label{tab:global-eg}
    \small
    \setlength{\tabcolsep}{7pt}
    \begin{tabular}{@{}lrrrrr@{}}
        \toprule
        & \multicolumn{2}{c}{XGBoost}
        & \multicolumn{3}{c}{MLP} \\
        \cmidrule(lr){2-3}
        \cmidrule(lr){4-6}
        Feature & Ours & TreeSHAP & Ours & KernelSHAP & DeepSHAP \\
        \midrule
        $X_{1}$ & 13.6437 & 13.4119 & 13.2043 & 12.1318 & 11.5254 \\
        $X_{2}$ & 13.7234 & 13.4129 & 13.1554 & 12.2072 & 11.8409 \\
        $X_{3}$ & 13.5724 & 14.0789 & 13.6184 & 13.8353 & 14.0872 \\
        $X_{4}$ & 13.5033 & 12.7727 & 13.0518 & 12.9702 & 13.0222 \\
        $X_{5}$ & 0.1487 & 0.4979 & 0.0882 & 1.1253 & 1.5344 \\
        $X_{6}$ & 0.3914 & 0.5186 & 0.4190 & 0.7438 & 1.2605 \\
        $X_{7}$ & 0.4831 & 0.5082 & 0.1860 & 0.7719 & 1.1540 \\
        $X_{8}$ & 0.6050 & 0.6157 & 0.2197 & 0.6337 & 1.0468 \\
        $X_{9}$ & 10.4042 & 10.5454 & 11.0589 & 11.1615 & 11.2422 \\
        $X_{10}$ & 11.4173 & 11.3722 & 11.7994 & 11.7145 & 11.4116 \\
        $X_{11}$ & 11.3252 & 11.6519 & 11.7348 & 11.1255 & 10.6959 \\
        $X_{12}$ & 10.7824 & 10.6137 & 11.4641 & 11.5791 & 11.1788 \\
        \bottomrule
    \end{tabular}
\end{table}

\begin{table}[!htbp]
    \centering
    \caption{Global feature importance (\%) on Pima Indian Diabetes
    (\textbf{PI}).}
    \label{tab:global-pi}
    \small
    \setlength{\tabcolsep}{7pt}
    \begin{tabular}{@{}lrrrrr@{}}
        \toprule
        & \multicolumn{2}{c}{XGBoost}
        & \multicolumn{3}{c}{MLP} \\
        \cmidrule(lr){2-3}
        \cmidrule(lr){4-6}
        Feature & Ours & TreeSHAP & Ours & KernelSHAP & DeepSHAP \\
        \midrule
        $X_{1}$ & 5.5059 & 4.9248 & 12.2998 & 12.9708 & 12.8868 \\
        $X_{2}$ & 35.2554 & 33.4458 & 36.5302 & 36.1625 & 34.8230 \\
        $X_{3}$ & 3.1025 & 3.6014 & 7.4437 & 7.7228 & 8.1034 \\
        $X_{4}$ & 1.8986 & 4.0152 & 1.9596 & 1.6997 & 2.2543 \\
        $X_{5}$ & 4.6394 & 4.8442 & 1.8499 & 2.1823 & 2.4806 \\
        $X_{6}$ & 19.3457 & 18.7768 & 20.4796 & 20.7625 & 20.8446 \\
        $X_{7}$ & 12.7497 & 12.6906 & 9.7095 & 9.1647 & 9.1810 \\
        $X_{8}$ & 17.5028 & 17.7012 & 9.7277 & 9.3346 & 9.4262 \\
        \bottomrule
    \end{tabular}
\end{table}

\begin{table}[!htbp]
    \centering
    \caption{Spearman rank correlations between our global feature
    importances and those of SHAP-based explainers. Comparisons use
    XGBoost for TreeSHAP and MLPs for KernelSHAP and DeepSHAP.
    The mean is computed with equal weight for each dataset.}
    \label{tab:global-spearman}
    \small
    \setlength{\tabcolsep}{9pt}
    \begin{tabular}{@{}lrrr@{}}
        \toprule
        & XGBoost & \multicolumn{2}{c}{MLP} \\
        \cmidrule(lr){2-2}
        \cmidrule(lr){3-4}
        Dataset & vs.\ TreeSHAP & vs.\ KernelSHAP & vs.\ DeepSHAP \\
        \midrule
        \textbf{BS} & 0.7619 & 0.8571 & 0.8333 \\
        \textbf{CH} & 0.9762 & 0.9762 & 0.9762 \\
        \textbf{CI} & 0.8857 & 0.9736 & 0.9692 \\
        \textbf{EG} & 0.9650 & 0.8881 & 0.8951 \\
        \textbf{PI} & 0.9762 & 0.9762 & 0.9762 \\
        \midrule
        \textbf{Mean}
        & \textbf{0.9130}
        & \textbf{0.9343}
        & \textbf{0.9300} \\
        \bottomrule
    \end{tabular}
\end{table}

\section{Additional Figures}\label{appendix:additional_figures}

\subsection{Bike Sharing Dataset}

\begin{figure}[H]
  \centering
  \includegraphics[width=0.9\textwidth]{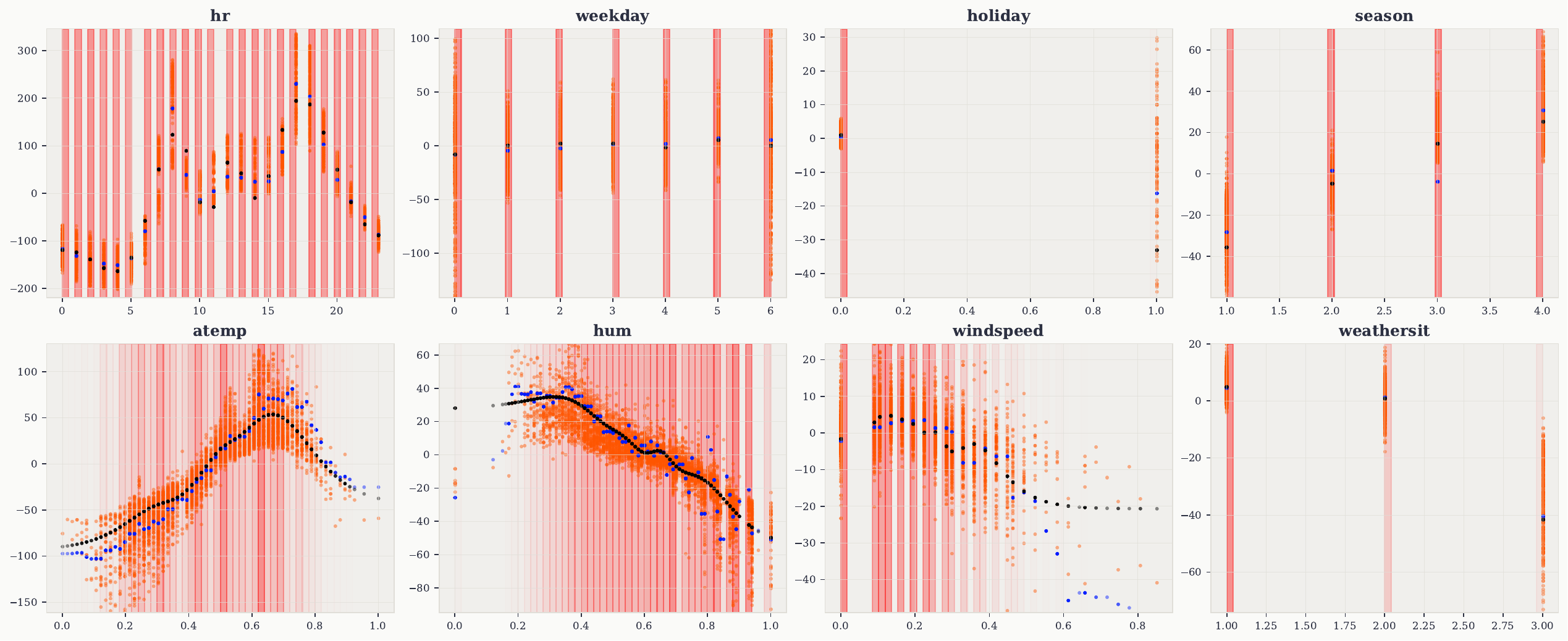}
  \caption{Estimated main effects on \emph{Bike Sharing}: our method (black) vs \textcolor{treehfd}{TreeHFD (main effects)} and \textcolor{treeshap}{TreeSHAP} on a trained XGB.}
  \label{fig:BS_tree_all}
\end{figure}

\begin{figure}[H]
  \centering
  \includegraphics[width=0.9\textwidth]{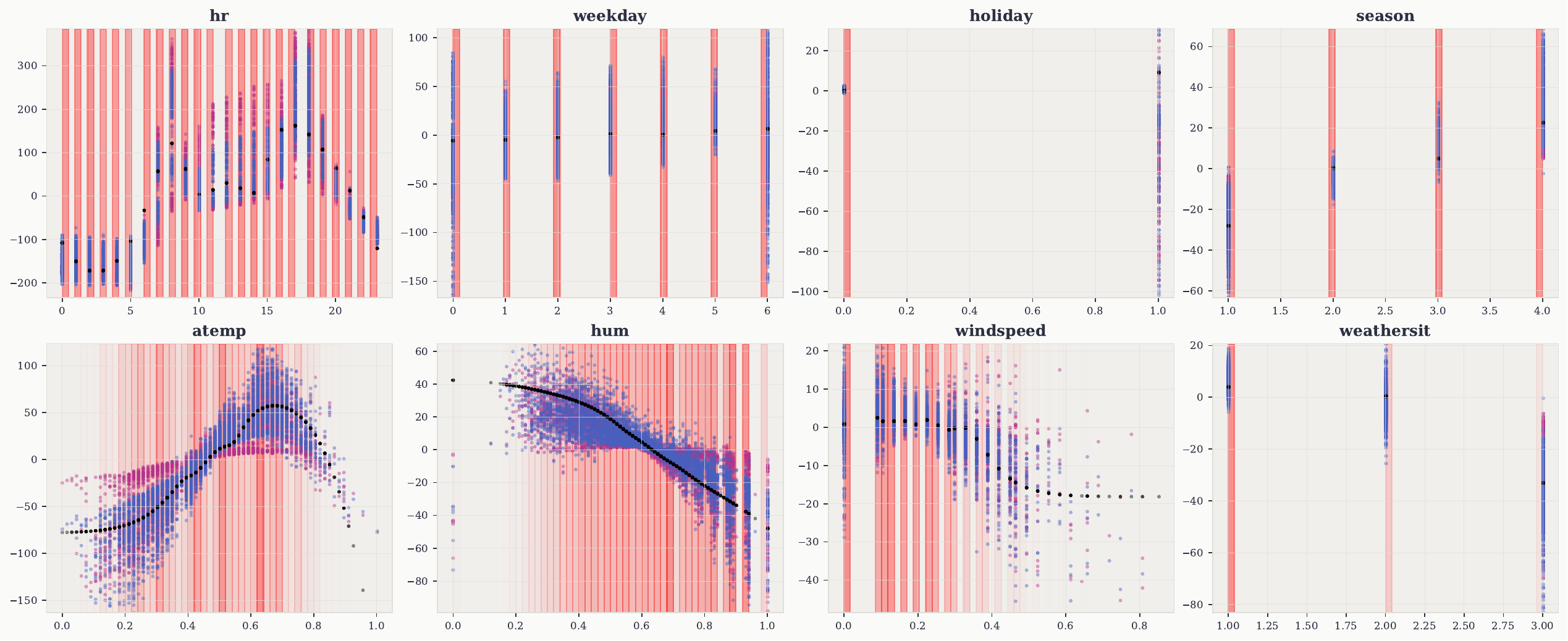}
  \caption{Estimated main effects on \emph{Bike Sharing}: our method (black) vs \textcolor{kernelshap}{KernelSHAP} and \textcolor{deepshap}{DeepSHAP} on a trained MLP.}
  \label{fig:BS_mlp_all}
\end{figure}

\begin{figure}[H]
  \centering
  \includegraphics[width=0.9\textwidth]{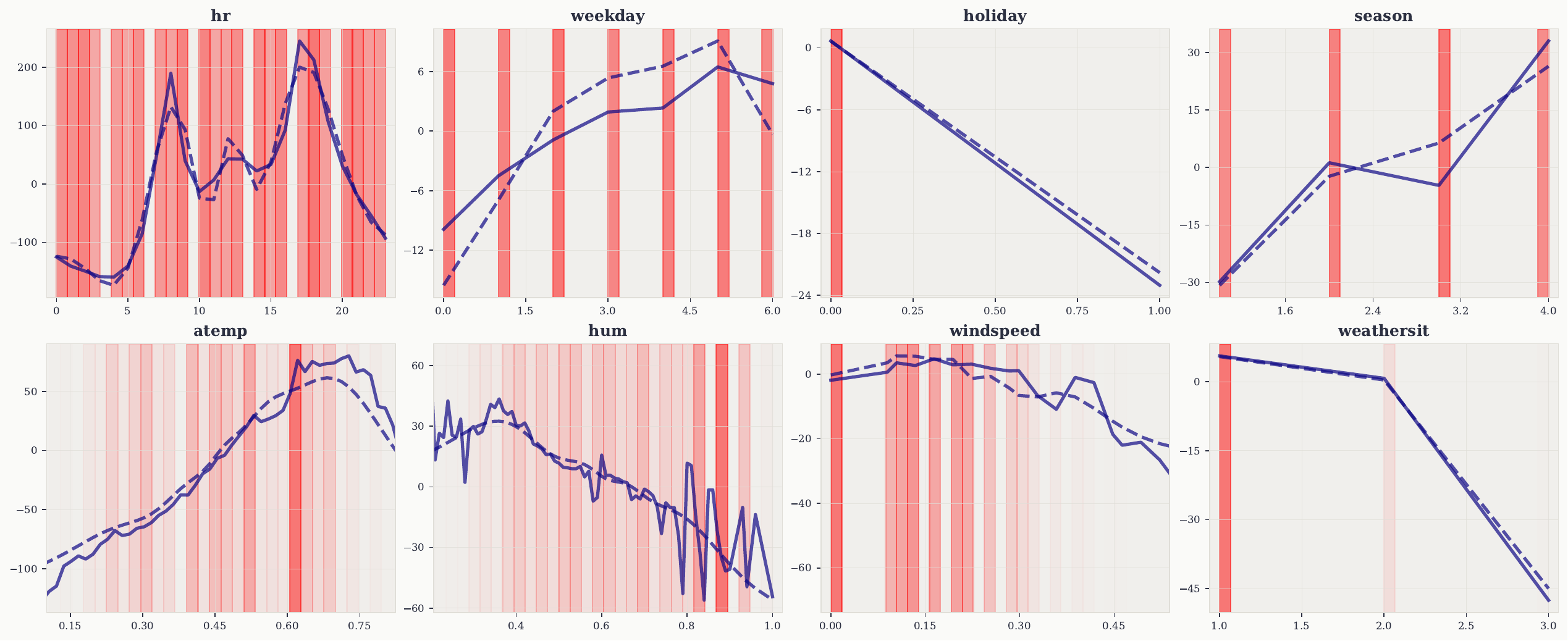}
  \caption{Comparison of native main effects from an EBM
           with those recovered by our method on
           \emph{Bike Sharing}.
           \textcolor{ebmcolor}{EBM} (solid) vs
           \textcolor{ebmcolor}{our method} (dashed).}
  \label{fig:BS_ebm}
\end{figure}

\subsection{California Housing Dataset}

\begin{figure}[H]
  \centering
  \includegraphics[width=1.00\textwidth]{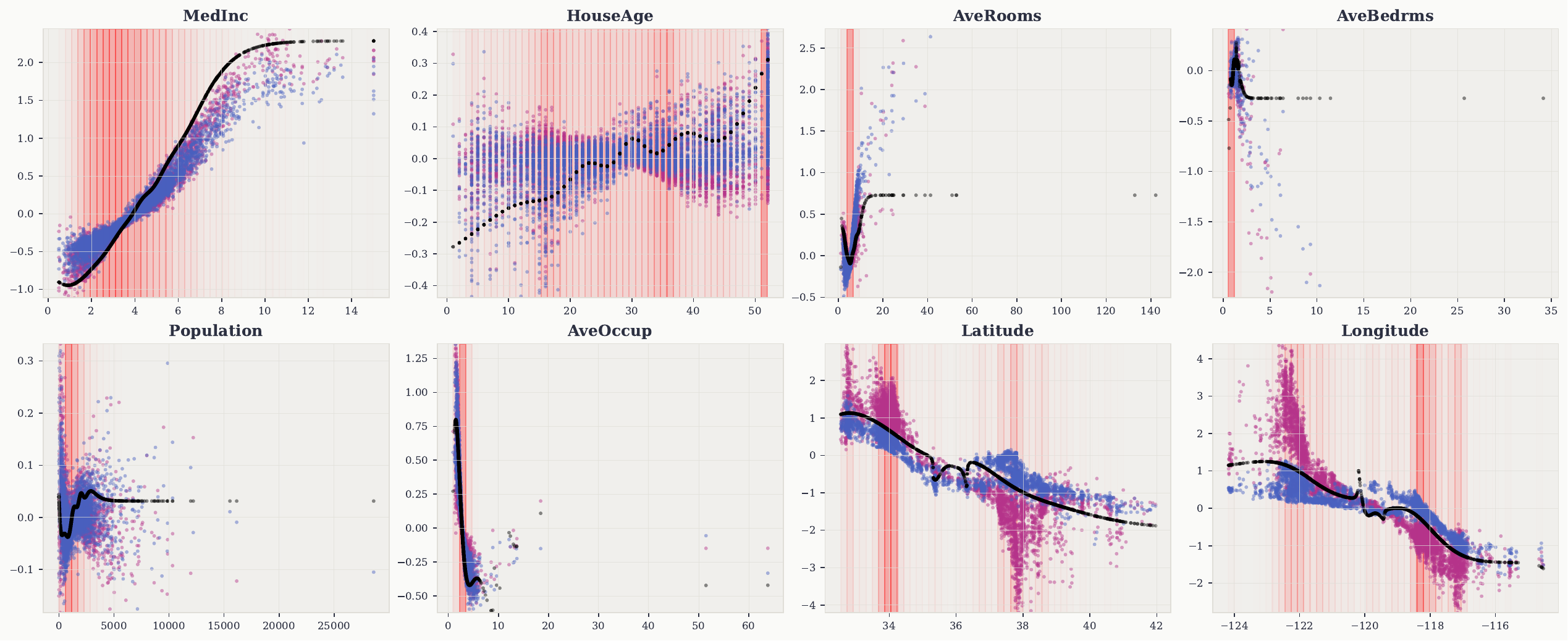}
  \caption{Estimated main effects on \emph{California Housing}: our method (black) vs \textcolor{kernelshap}{KernelSHAP} and \textcolor{deepshap}{DeepSHAP} on a trained MLP.}
  \label{fig:CH_mlp_all}
\end{figure}

\begin{figure}[H]
  \centering
  \includegraphics[width=1.00\textwidth]{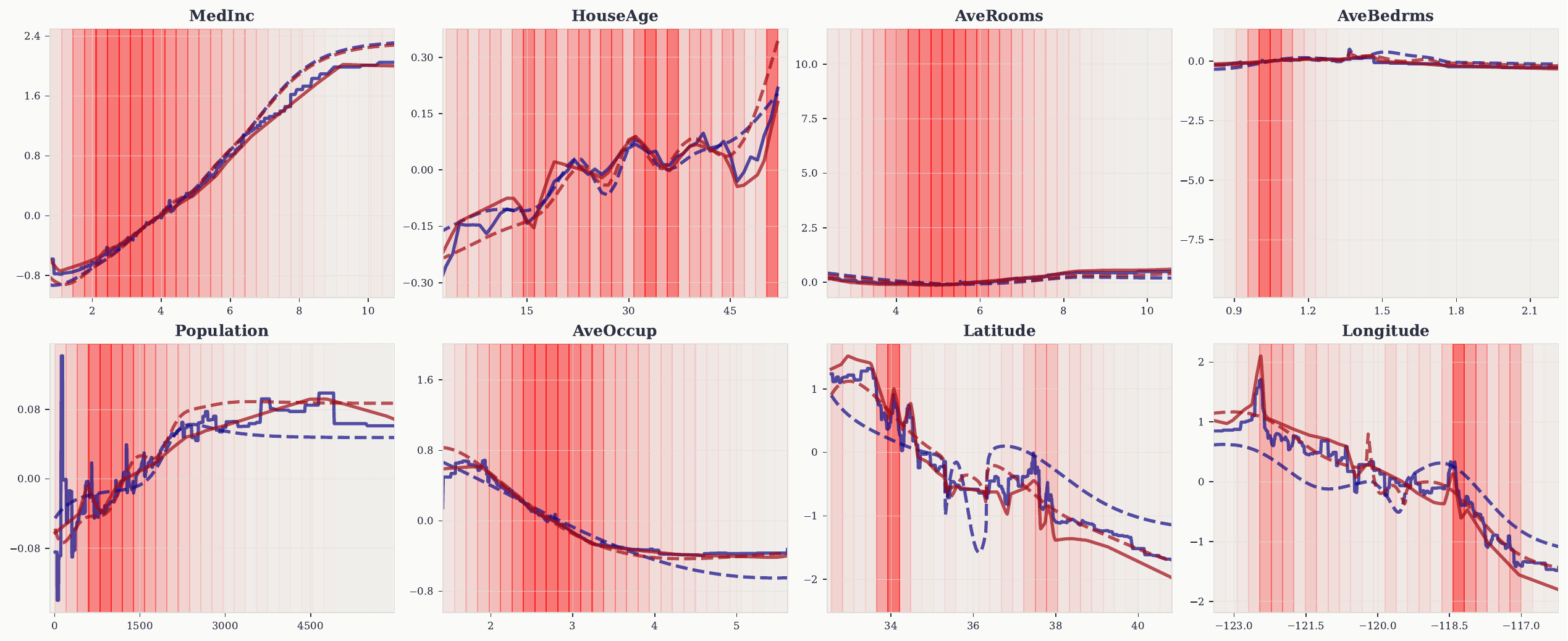}
  \caption{Comparison of native main effects from an EBM
           with those recovered by our method on
           \emph{California Housing}.
           \textcolor{ebmcolor}{EBM} (solid) vs
           \textcolor{ebmcolor}{our method} (dashed);
           \textcolor{namcolor}{NAM} (solid) vs
           \textcolor{namcolor}{our method} (dashed).}
  \label{fig:CH_ebm_nam}
\end{figure}

\subsection{Census Income Dataset}

\begin{figure}[H]
  \centering
  \includegraphics[width=0.8\textwidth]{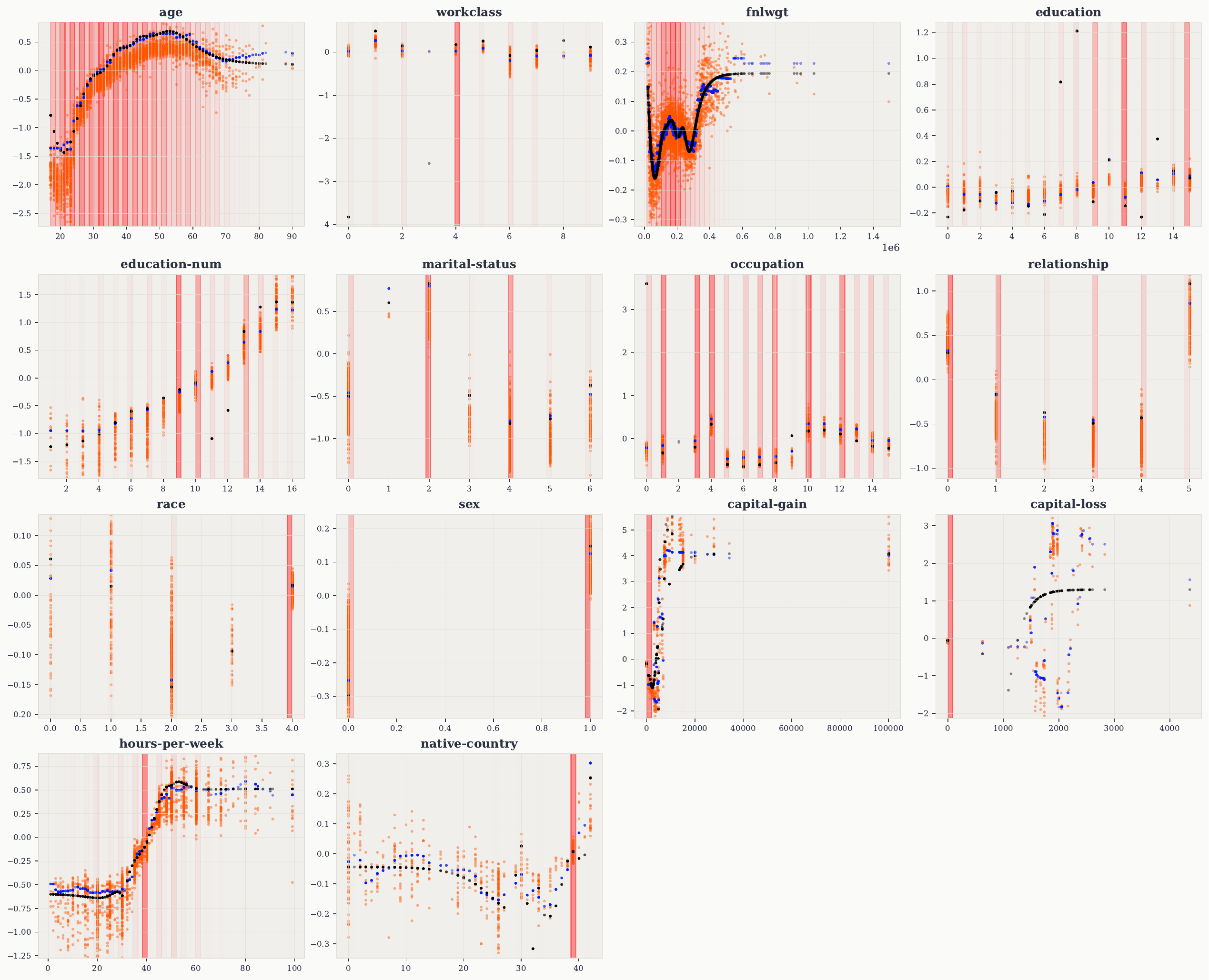}
  \caption{Estimated main effects on \emph{Census Income}: our method (black) vs \textcolor{treehfd}{TreeHFD (main effects)} and \textcolor{treeshap}{TreeSHAP} on a trained XGB.}
  \label{fig:CI_tree_all}
\end{figure}

\begin{figure}[H]
  \centering
  \includegraphics[width=0.8\textwidth]{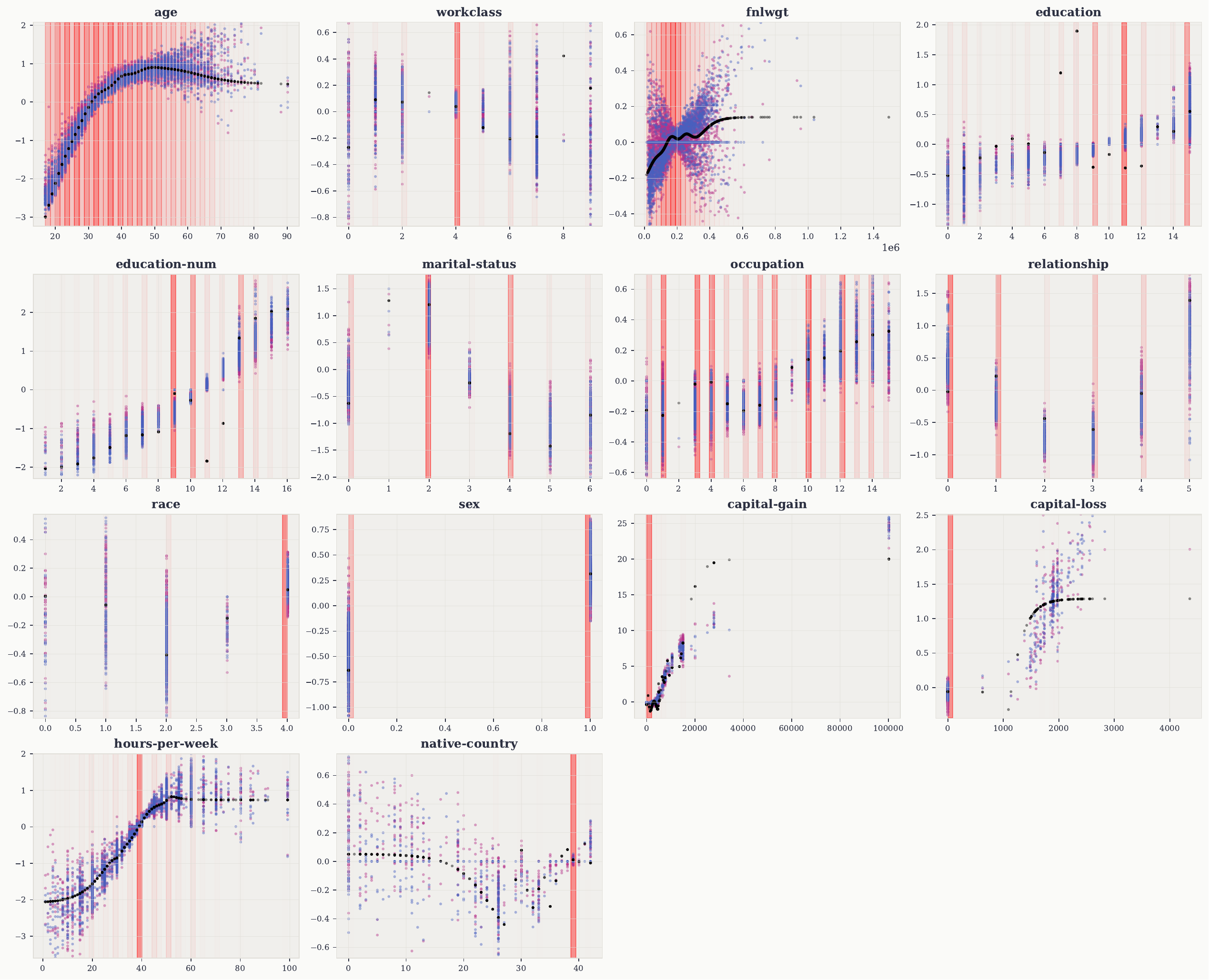}
  \caption{Estimated main effects on \emph{Census Income}: our method (black) vs \textcolor{kernelshap}{KernelSHAP} and \textcolor{deepshap}{DeepSHAP} on a trained MLP.}
  \label{fig:CI_mlp_all}
\end{figure}

\begin{figure}[H]
  \centering
  \includegraphics[width=0.8\textwidth]{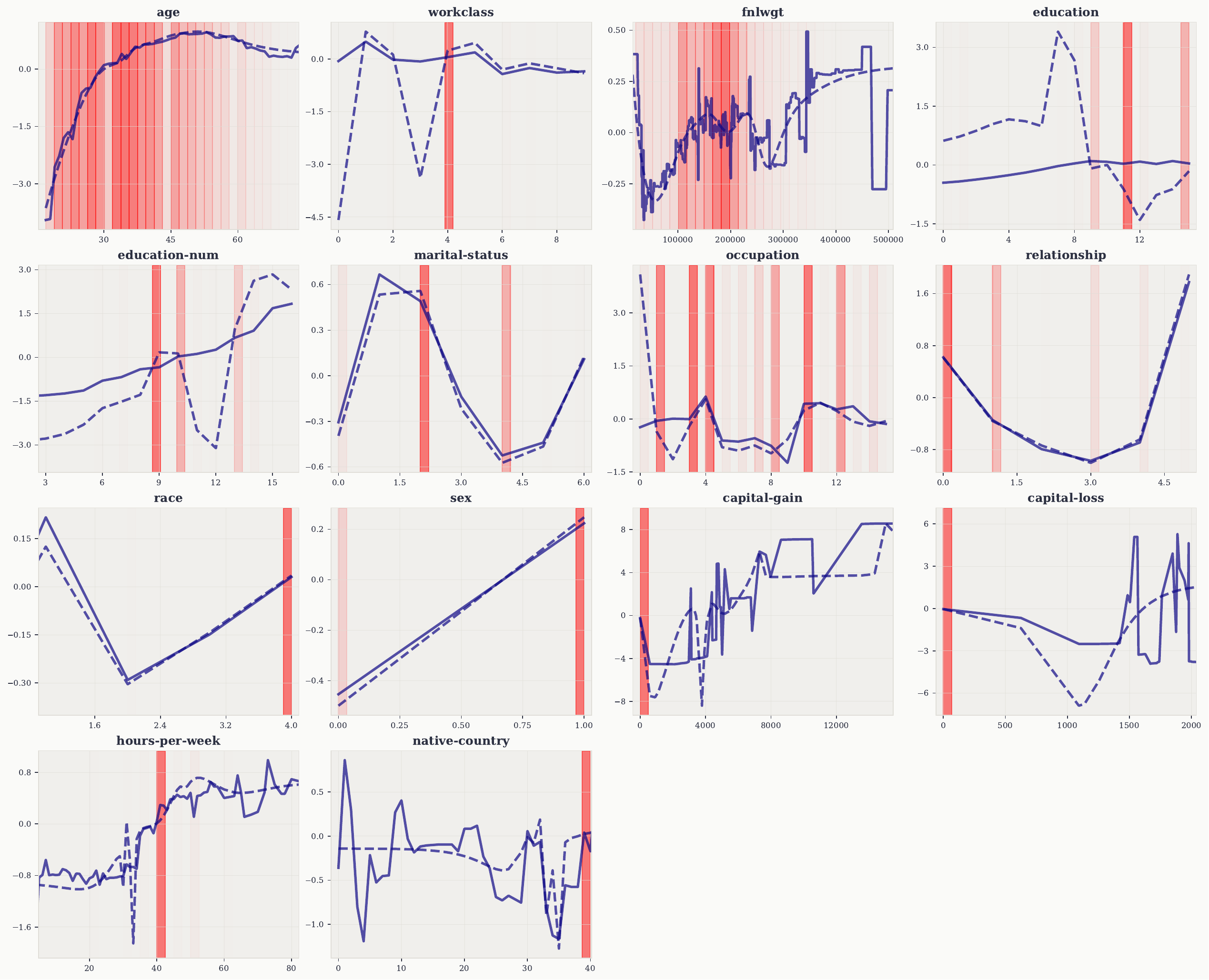}
  \caption{Comparison of native main effects from an EBM
           with those recovered by our method on
           \emph{Census Income}.
           \textcolor{ebmcolor}{EBM} (solid) vs
           \textcolor{ebmcolor}{our method} (dashed).}
  \label{fig:CI_ebm}
\end{figure}

\subsection{Electrical Grid Dataset}

\begin{figure}[H]
  \centering
  \includegraphics[width=1.0\textwidth]{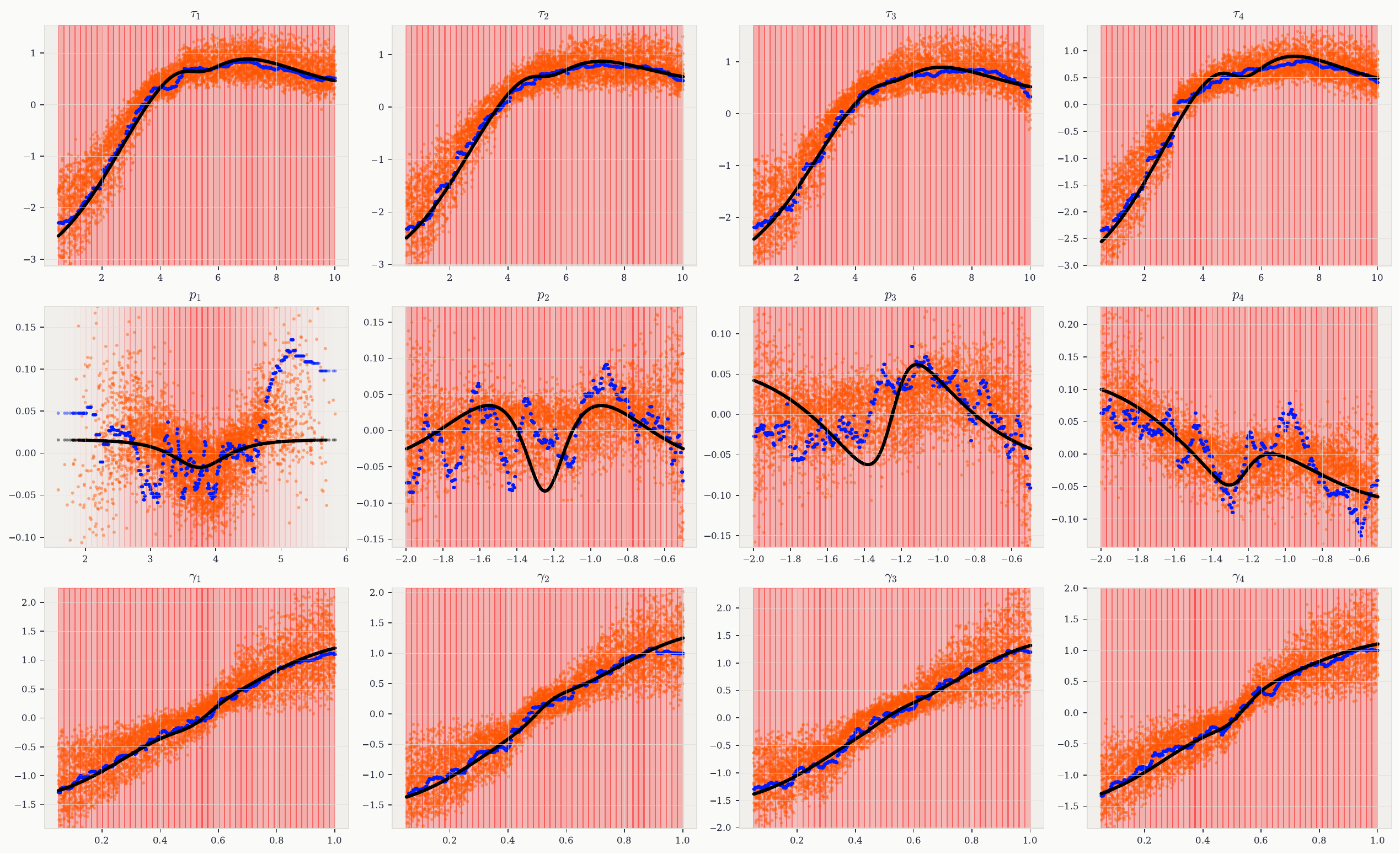}
  \caption{Estimated main effects on \emph{Electrical Grid}: our method (black) vs \textcolor{treehfd}{TreeHFD (main effects)} and \textcolor{treeshap}{TreeSHAP} on a trained XGB.}
  \label{fig:EG_tree_all}
\end{figure}

\begin{figure}[H]
  \centering
  \includegraphics[width=1.0\textwidth]{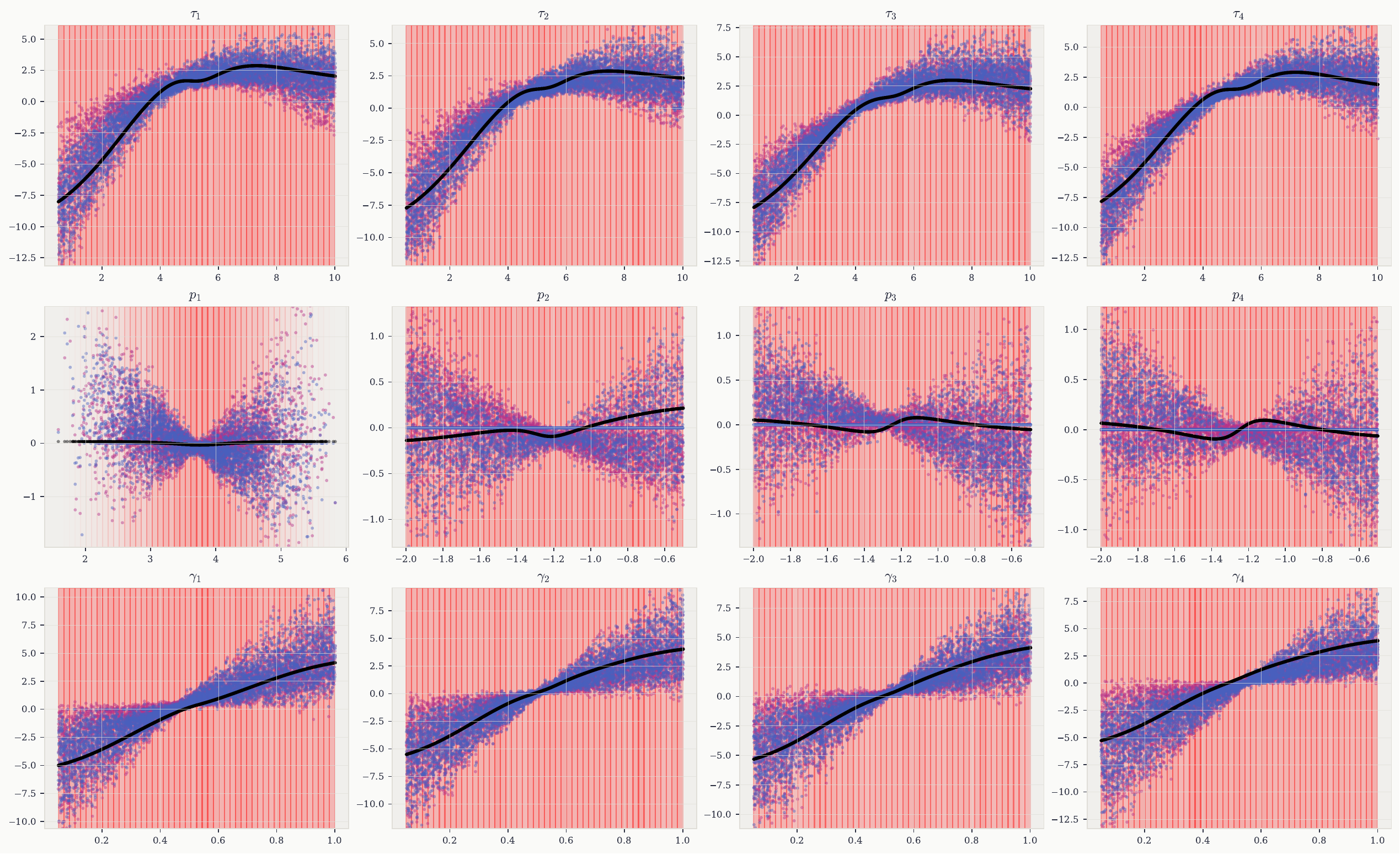}
  \caption{Estimated main effects on \emph{Electrical Grid}: our method (black) vs \textcolor{kernelshap}{KernelSHAP} and \textcolor{deepshap}{DeepSHAP} on a trained MLP.}
  \label{fig:EG_mlp_all}
\end{figure}

\begin{figure}[H]
  \centering
  \includegraphics[width=1.0\textwidth]{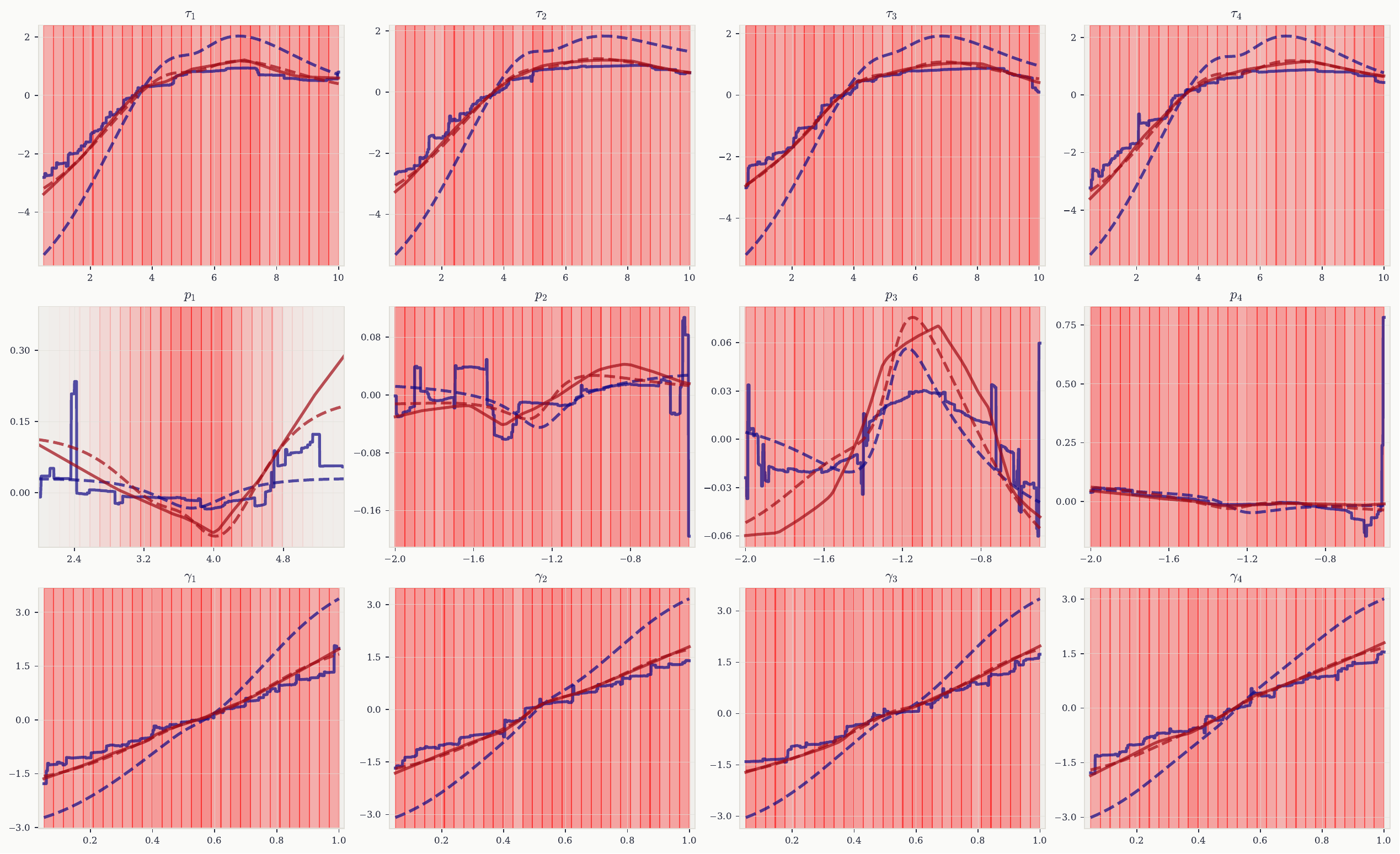}
  \caption{Comparison of native main effects from an EBM
           with those recovered by our method on
           \emph{Electrical Grid}.
           \textcolor{ebmcolor}{EBM} (solid) vs
           \textcolor{ebmcolor}{our method} (dashed);
           \textcolor{namcolor}{NAM} (solid) vs
           \textcolor{namcolor}{our method} (dashed).}
  \label{fig:EG_ebm_nam}
\end{figure}

\subsection{Diabetes Dataset}

\begin{figure}[H]
  \centering
  \includegraphics[width=1.0\textwidth]{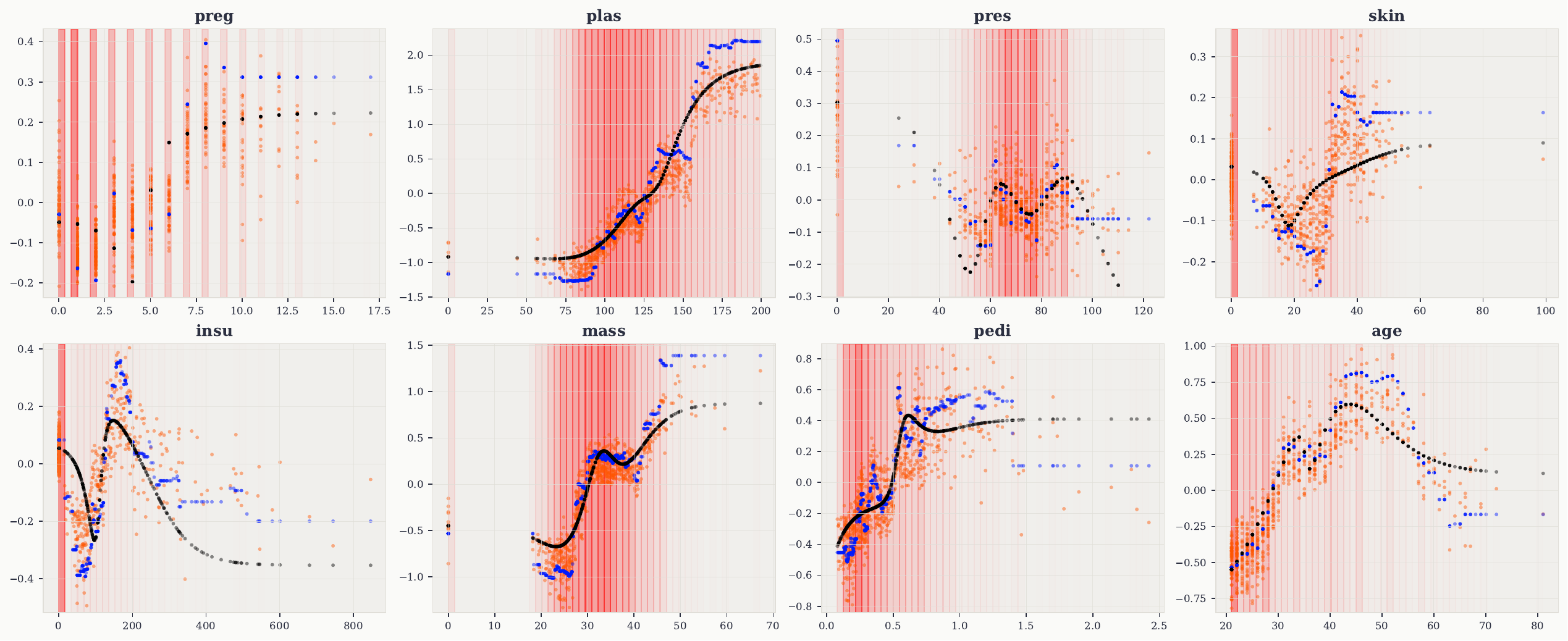}
  \caption{Estimated main effects on \emph{Diabetes}: our method (black) vs \textcolor{treehfd}{TreeHFD (main effects)} and \textcolor{treeshap}{TreeSHAP} on a trained XGB.}
  \label{fig:PI_tree_all}
\end{figure}

\begin{figure}[H]
  \centering
  \includegraphics[width=1.0\textwidth]{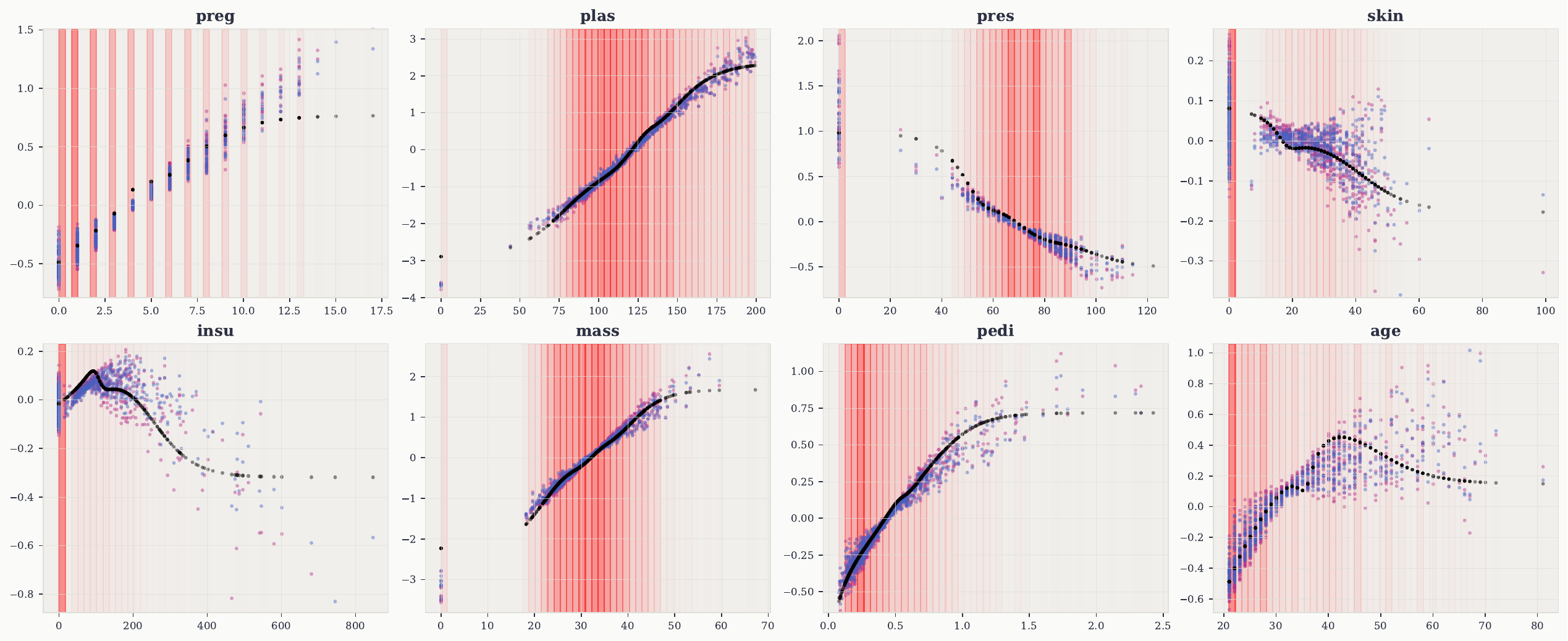}
  \caption{Estimated main effects on \emph{Diabetes}: our method (black) vs \textcolor{kernelshap}{KernelSHAP} and \textcolor{deepshap}{DeepSHAP} on a trained MLP.}
  \label{fig:PI_mlp_all}
\end{figure}

\stopcontents[appendices]

\end{document}